\documentclass[twoside,11pt]{article}

\usepackage[preprint]{jmlr2e}

\usepackage[utf8]{inputenc}
\usepackage[T1]{fontenc}
\usepackage{amsmath}
\usepackage{amsfonts}
\usepackage{amssymb}
\usepackage{algorithm}
\usepackage{algpseudocode}
\usepackage{array}
\usepackage{booktabs}
\usepackage{multirow}
\usepackage{nicefrac}
\usepackage{microtype}
\usepackage{xcolor}
\usepackage{pifont}
\usepackage{wrapfig}
\usepackage{caption}
\usepackage{enumitem}
\usepackage{lastpage}
\usepackage{url}
\usepackage{graphicx}
\usepackage{mathtools}

\newtheorem{assumption}[theorem]{Assumption}
\providecommand{\qedsymbol}{$\blacksquare$}
\renewcommand{\qedsymbol}{$\blacksquare$}

\jmlrheading{1}{2026}{1-\pageref{LastPage}}{6/26}{--}{--}{Jiawei Tang, Yuheng Jia, Xinyan Du, Hui Liu, and Junhui Hou}
\ShortHeadings{Variational Inference for Evidential Deep Learning}{Tang et al.}
\firstpageno{1}

\begin{document}

\title{Variational Inference for Evidential Deep Learning}

\author{\name Jiawei Tang \email jw\_tang@seu.edu.cn \\
       \addr School of Computer Science and Engineering\\
       Southeast University\\
       Nanjing 210096, China
              \AND
       \name Yuheng Jia \email yhjia@cseu.edu.cn \\
       \addr School of Computer Science and Engineering\\
       Southeast University\\
       Nanjing 210096, China
       \AND
       \name Xinyan Du \email dxy-sjcx@seu.edu.com \\
       \addr School of Computer Science and Engineering\\
       Southeast University\\
       Nanjing 210096, China
       \AND
       \name Hui Liu \email h2liu@sfu.edu.hk \\
       \addr School of Computing Information Sciences\\
       Saint Francis University\\
       Hong Kong SAR, China
       \AND
       \name Junhui Hou \email jh.hou@cityu.edu.hk \\
       \addr Department of Computer Science\\
       City University of Hong Kong\\
       Hong Kong SAR, China
}

\maketitle

\begin{abstract}
While deep neural networks achieve remarkable performance, their tendency to produce overconfident predictions remains a critical issue. Evidential deep learning (EDL) mitigates this by modeling predictions as a Dirichlet distribution over class probabilities, thereby explicitly quantifying epistemic uncertainty. However, we identify two fundamental limitations in conventional EDL: first, its Kullback--Leibler penalty only suppresses the evidence of negative classes, which produces excessively strong evidence and consequently degrades the model's ability to quantify uncertainty; second, it lacks a theoretical guarantee for setting the Dirichlet parameter $\boldsymbol{\alpha} = \mathbf{e} + \mathbf{1}$. In this paper, we propose a mathematically principled framework, called Variational Inference EDL (VI-EDL). By reformulating evidential learning through the lens of variational inference, we derive an Evidence Lower Bound that prevents the evidence from growing excessively. Furthermore, we rigorously establish a generalization bound, revealing how predicted uncertainty, feature representations, and network complexity affect this bound, and proving why setting $\boldsymbol{\alpha} = \mathbf{e} + \mathbf{1}$  minimizes it. Extensive experiments on standard visual and medical datasets demonstrate that VI-EDL achieves state-of-the-art performance, exhibiting superior performance in out-of-distribution detection, noise detection, and autonomous driving scenarios. The code is available at \url{https://github.com/seutjw/VI-EDL}.
\end{abstract}


\section{Introduction}

Deep Neural Networks (DNNs) have achieved success in a wide range of predictive tasks \citep{DNN1, DNN2}. However, their tendency to produce overconfident point-estimates \citep{of1, of2}---even when encountering out-of-distribution (OOD) samples \citep{ood1, ood2, ood3, ood4, ood5}---remains a critical bottleneck for deployment in safety-critical domains like autonomous driving \citep{drive1, drive2, drive3} and biomedical applications \citep{med1}. While Bayesian Neural Networks \citep{BNN} and MC-Dropout \citep{MC} offer some solutions, they typically require computationally expensive multiple sampling. To address this, Evidential Deep Learning (EDL) \citep{EDL} has emerged as a promising paradigm. Instead of applying a standard softmax function to output deterministic probabilities, EDL restructures the prediction pipeline into an evidence-based chain. Specifically, EDL reformulates the standard classification paradigm by estimating a class probability distribution rather than providing a single point estimate. 

EDL models the predicted class probability distribution as a Dirichlet distribution \citep{SL, DS}. Given an input $\mathbf{x}$, a neural network $f_\theta$ is employed to extract non-negative evidence vector $\mathbf{e} = [e_1, \dots, e_K] \geq 0$ for $K$ classes, typically utilizing an activation function such as Softplus. This evidence vector is then used to parameterize the Dirichlet distribution $\text{Dir}(\mathbf{p} | \boldsymbol{\alpha})$, where $\boldsymbol{\alpha}$ is the Dirichlet parameters $\boldsymbol{\alpha}=[\alpha_1,...,\alpha_K]$ with each element defining as $\alpha_k = e_k + 1, \forall k\in[1,2,...,K]$ and the expected class probability vector $\mathbf{p}=[\hat{p}_1,...,\hat{p}_K]$. Accordingly, the uncertainty is quantified by $u = K/S$ and the expected class probabilities are $\hat{p}_k = \alpha_k / S$, where $k\in[1,2,...,K]$ and $S = \sum_{k=1}^K \alpha_k$.

To train the model, standard EDL minimizes an expected Mean Squared Error (MSE) risk. This term measures the expected discrepancy between the ground-truth label vector $\mathbf{y}=[y_1,...,y_K]\in\{0,1\}^K$ and the expected class probability vector $\mathbf{p}$. By evaluating the expectation over the simplex, this objective implicitly minimizes the variance of the predictions. Furthermore, to prevent the model from generating misleading evidence for incorrect classes, a Kullback-Leibler (KL) divergence term is introduced. It penalizes the divergence between the predicted distribution (excluding the evidence of the true class) and a flat dirichlet distribution $\text{Dir}(\mathbf{p} | \mathbf{1})$, where $\mathbf{1}$ represents an all-one vector.

The overall loss function of standard EDL is formulated as:
\begin{equation}
\begin{aligned}
    \mathcal{L}_{EDL} =& \underbrace{\mathbb{E}_{\mathbf{p} \sim \text{Dir}(\boldsymbol{\alpha})} \left[ \| \mathbf{y} - \mathbf{p} \|_2^2 \right]}_{\text{Expected MSE}} +                                                         \lambda_t \cdot \underbrace{D_{KL} \left( \text{Dir}(\mathbf{p} | \tilde{\boldsymbol{\alpha}}) \parallel \text{Dir}(\mathbf{p} | \mathbf{1}) \right)}_{\text{KL divergence penalty}},
\end{aligned}
\end{equation}
where $\lambda_t$ is an annealing coefficient that gradually ramps up the KL penalty during the initial training epochs, and $\tilde{\boldsymbol{\alpha}} = \mathbf{y} + (\mathbf{1} - \mathbf{y}) \odot \boldsymbol{\alpha}$ represents the modified Dirichlet parameters after masking out the true class evidence, where $\mathbf{1}\in\mathbb{R}^K$ represents an all-one vector.

Despite its conceptual elegance, the conventional EDL framework suffers from two drawbacks:

\begin{itemize}[leftmargin=*, nosep]
    \item{In EDL, the KL divergence term only suppresses the evidence of negative classes while leaving the positive class unconstrained. This inherently motivates the network to take an optimization shortcut, blindly inflating feature and weight magnitudes to produce excessively high target evidence. Driven by the MSE term, the network tends to act as a numerical amplifier and inevitably maps even non-semantic random noise to massive spurious evidence. This directly decreases the model's ability to quantify uncertainty.}
    \item{EDL maps the generated evidence $\mathbf{e}$ to the Dirichlet parameters $\boldsymbol{\alpha}$ using a heuristically defined function $\boldsymbol{\alpha} = \mathbf{e} + \mathbf{1}$. Despite its effectiveness empirically, this formulation lacks a principled mathematical justification.}
\end{itemize}

In this work, we propose resolving these two limitations by reconstructing the EDL framework in a probabilistic framework. We transform the evidence extraction process into a variational inference (VI) \citep{VI} problem and derive a rigorous Evidence Lower Bound (ELBO) \citep{VAE}, organically restoring the Bayesian legitimacy of the EDL paradigm. Moreover, the form of the ELBO also naturally dictates a global regularization across all classes, structurally preventing the magnitude-induced overconfidence. In addition, we also design a cosine prototype layer for the evidence network to further control the magnitude of generated evidence. 

Furthermore, our framework provides a definitive theoretical closure to the heuristic $\boldsymbol{\alpha} = \mathbf{e} + \mathbf{1}$ mapping. By imposing a generalized Dirichlet prior $\text{Dir}(\boldsymbol{\lambda})$ in the probabilistic model, our variational formulation rigorously derives the posterior parameterization as $\boldsymbol{\alpha} = \mathbf{e} + \boldsymbol{\lambda}$ via Bayesian conjugate updating. Consequently, we conduct a rigorous generalization bound analysis, theoretically proving that setting the prior parameter to $\boldsymbol{\lambda} = \mathbf{1}$ optimally minimizes the generalization error bound. This finding retrospectively validates the empirical design of standard EDL while firmly grounding it in statistical learning theory.

The main contributions of this work are summarized as follows:
\begin{itemize}
    \item 
    We provide a novel probabilistic derivation of the evidential objective using Variational Inference (VI-EDL). By formulating the evidence generation as a conjugate update, we replace heuristic terms with a mathematically rigorous ELBO, which guarantees the theoretical soundness and interpretability of our model. Along with the cosine prototype layer, it also prevents the evidence of all classes from growing excessively.
    \item 
    We analytically derive a generalization bound for our proposed model and provide several theoretical insights in Section~\ref{insights}. Specifically, we show that setting $\boldsymbol{\alpha} = \mathbf{e} + \mathbf{1}$ minimizes the bound (\hyperref[insight:alpha]{\textbf{Insight 2}}), and further reveal how the predicted uncertainty (\hyperref[insight:uncertainty]{\textbf{Insight 3}}), feature complexity, and network complexity (\hyperref[insight:complexity]{\textbf{Insight 4}}) affect this bound. Therefore, we fill a fundamental theoretical gap in standard EDL.
    \item 
    We evaluate our framework on various benchmarks, including standard visual datasets and medical datasets. Experimental results demonstrate that VI-EDL significantly outperforms state-of-the-art evidential baselines, achieving superior performance in out-of-distribution (OOD) detection, noise detection, and autonomous driving scenarios.
\end{itemize}

The remainder of this paper is organized as follows. Section \ref{sec:related_work} briefly reviews the related literature on uncertainty estimation and evidential deep learning. Section \ref{sec:method} elaborates on the proposed Variational Inference Evidential Deep Learning (VI-EDL) framework, including the derivation of the ELBO and the design of the cosine prototype layer. Section \ref{sec:theory} provides a rigorous theoretical analysis of our approach, establishing a generalization bound and delving into its insights. Section \ref{sec:experiments} presents extensive experimental results, ablation studies, and robustness analyses across multiple benchmarks to demonstrate the superiority of the proposed method. Finally, Section \ref{sec:conclusion} concludes this paper and outlines potential avenues for future research.

\section{Related Work}
\label{sec:related_work}
Deep neural networks often produce overconfident predictions even when their accuracy is high, especially under distribution shift or on semantically invalid inputs \citep{of1,of2}. This has motivated a broad line of research on predictive uncertainty estimation and confidence calibration. Bayesian neural networks explicitly place distributions over network parameters, but usually require approximate inference and additional computational cost \citep{BNN}. MC-Dropout interprets dropout as a Bayesian approximation and estimates uncertainty by multiple stochastic forward passes \citep{MC}. Deep ensembles provide a simple and strong alternative by aggregating independently trained models, yielding competitive calibration and OOD robustness at the cost of multiple model evaluations \citep{DE}. In parallel, OOD detection has been extensively studied as a way to identify test samples that deviate from the training distribution \citep{ood2}. Representative approaches use maximum softmax confidence scores \citep{ood1}, temperature-scaled confidence with input perturbations \citep{ood5}, feature-space distances \citep{ood3}, or energy-based scores \citep{ood4}. Although these methods improve reliability, many of them either require post-hoc calibration, multiple forward passes, or auxiliary scoring functions. In contrast, evidential models aim to quantify uncertainty deterministically in a single forward pass.

Evidential Deep Learning (EDL) provides a deterministic uncertainty quantification framework by mapping neural network outputs to the parameters of a Dirichlet distribution over class probabilities \citep{EDL}. Built on Subjective Logic \citep{SL} and Dempster--Shafer theory \citep{DS}, EDL interprets non-negative network outputs as class-wise evidence and estimates epistemic uncertainty through the total Dirichlet strength. Following this paradigm, several variants have been proposed to improve uncertainty estimation and training stability. I-EDL introduces Fisher information to assess the informativeness of evidence and dynamically reweight learning objectives \citep{IEDL}. Re-EDL revisits essential and nonessential design choices in EDL, treating the Dirichlet prior weight as an adjustable hyperparameter and simplifying the optimization objective \citep{REEDL}. F-EDL further extends the Dirichlet assumption by introducing a flexible Dirichlet distribution to obtain a more expressive uncertainty representation \citep{FEDL}.

Beyond methodological developments, evidential models have been applied to medical image analysis and retrieval \citep{EDL_Medical_Seg,EDL_Medical_Ret}, autonomous driving and multi-modal sensor fusion \citep{EDL_Autonomous_Fusion1,EDL_Autonomous_Fusion2}, and molecular property prediction \citep{soleimany2021evidential}. However, most existing EDL methods still inherit two fundamental limitations of standard EDL: the evidence generation process is not derived from a principled probabilistic objective, and the conventional KL penalty does not directly constrain the positive-class evidence, leaving the model vulnerable to magnitude-induced overconfidence.

While these EDL methods demonstrate empirical progress, they share a critical vulnerability: failing to impose any restriction on the unbounded growth of positive class evidence. Our proposed VI-EDL resolves this magnitude-induced overconfidence. By deriving the evidential objective strictly from the Evidence Lower Bound (ELBO) and enforcing a cosine prototype evidence layer, our framework bounds the maximum achievable evidence strictly, ensuring the reliability of evidential models in complex environments.

\section{Proposed Method}
\label{sec:method}

To reconstruct Evidential Deep Learning from the probabilistic perspective of Variational Inference (VI), we first formulate the classification task as a latent variable generative model and derive the Evidence Lower Bound (ELBO) objective. Subsequently, we introduce a distance-aware cosine prototype evidence network to generate the evidence.

\textbf{Notations.}
Let $\mathcal{X} \subset \mathbb{R}^{n \times d}$ be the input space for a $K$-class classification task. Given an input $\mathbf{x} \in \mathcal{X}$, the ground-truth label is represented as a one-hot vector $\mathbf{y}=[y_1,...,y_K] \in \{0, 1\}^K$, where $y_k = 1$ if the sample belongs to class $k$ and $y_k = 0$ otherwise. Let $\mathbf{p} = [p_1, \dots, p_K]$ represent a categorical class probability vector. To ensure valid probability semantics (sum up to 1), $\mathbf{p}$ must reside on the $(K-1)$-dimensional unit simplex, i.e., $\sum_{k=1}^K p_k = 1$ and $\forall k\in[1,2,...,K],p_k \geq 0 $.

\subsection{Variational Inference Framework for EDL}

From a probabilistic perspective, we treat the class probability distribution $\mathbf{p}$ as a latent variable. In a purely analytical Bayesian setting, we assign a Dirichlet prior over $\mathbf{p}$, denoted as $P(\mathbf{p}) = \text{Dir}(\mathbf{p}; \boldsymbol{\lambda}) \propto \prod_{k=1}^K p_k^{\lambda_k - 1}$. If one perfectly observes the categorical pseudo-counts $\mathbf{e} = [e_1, \dots, e_K]$ (i.e., the evidence of each class, where $\forall k\in[1,2,...,K], e_k\ge0$), the likelihood follows a Multinomial distribution $P(\mathbf{e}|\mathbf{p}) \propto \prod_{k=1}^K p_k^{e_k}$. According to Bayes' theorem, the exact posterior is formulated by multiplying the likelihood and the prior, where the exponents naturally add up:
\begin{equation}
    P(\mathbf{p}|\mathbf{e}) \propto P(\mathbf{e}|\mathbf{p}) P(\mathbf{p}) \propto \prod_{k=1}^K p_k^{(e_k + \lambda_k) - 1}.
\end{equation}
This Dirichlet-Multinomial conjugacy gracefully dictates that the exact posterior is $P(\mathbf{p}|\mathbf{e}) = \text{Dir}(\mathbf{p}; \mathbf{e} + \boldsymbol{\lambda})$. However, in our actual classification task, these explicit pseudo-counts are latent; we can only observe the raw input $\mathbf{x}$, rendering the true posterior computationally intractable. We apply variational inference to solve this by defining an approximate posterior $q_\phi(\mathbf{p}|\mathbf{x})$. To structurally ground this approximation in Bayesian principles, we design its variational family to explicitly mimic the additive posterior conjunction. Specifically, we employ a neural network $f_\phi$ to extract evidence $\mathbf{e} = f_\phi(\mathbf{x}) \in \mathbb{R}^K$ from the input, then parameterize the variational posterior using this extracted evidence combined with the prior. It is defined as:
\begin{equation}
q_\phi(\mathbf{p}|\mathbf{x}) = \text{Dir}(\mathbf{p}; \boldsymbol{\alpha}), \quad \text{where }\boldsymbol{\alpha} = \mathbf{e} + \boldsymbol{\lambda}.
\end{equation}

We optimize the model by maximizing the log marginal likelihood of the observed label $\mathbf{y}$. By introducing the variational posterior $q_\phi(\mathbf{p}|\mathbf{x})$, we derive the Evidence Lower Bound (ELBO) explicitly:
\begin{equation}
\begin{aligned}
    \log P(\mathbf{y}|\mathbf{x}) =& \log \int P(\mathbf{y}|\mathbf{p}) P(\mathbf{p}) d\mathbf{p} \\
    =& \log \int q_\phi(\mathbf{p}|\mathbf{x}) \frac{P(\mathbf{y}|\mathbf{p}) P(\mathbf{p})}{q_\phi(\mathbf{p}|\mathbf{x})} d\mathbf{p} \\
    \ge& \mathbb{E}_{q_\phi(\mathbf{p}|\mathbf{x})} \left[ \log \frac{P(\mathbf{y}|\mathbf{p}) P(\mathbf{p})}{q_\phi(\mathbf{p}|\mathbf{x})} \right] \\
    =& \mathbb{E}_{q_\phi(\mathbf{p}|\mathbf{x})} [\log P(\mathbf{y}|\mathbf{p})]-D_{KL}(q_\phi(\mathbf{p}|\mathbf{x}) \parallel P(\mathbf{p})).
\end{aligned}
\end{equation}

Maximizing this lower bound is mathematically equivalent to minimizing our variational loss function. To provide further flexibility in balancing the accuracy-uncertainty trade-off, we introduce a hyper-parameter $\beta$, leading to the final objective:
\begin{equation}
\begin{aligned}
    \mathcal{L}_{ELBO} =& \underbrace{-\mathbb{E}_{q_\phi(\mathbf{p}|\mathbf{x})}[\log P(\mathbf{y}|\mathbf{p})]}_{\text{Expected Negative Log-Likelihood}} + \beta \cdot \underbrace{D_{KL}(q_\phi(\mathbf{p}|\mathbf{x}) \parallel P(\mathbf{p}))}_{\text{Prior Regularization}}.
\label{ELBO}
\end{aligned}
\end{equation}

The role of the prior regularization term is to suppress the evidence across all classes, thereby mitigating overconfidence caused by excessively high evidence.

We assume that $\mathbf{p}$ is sampled from a Dirichlet prior distribution parameterized by a prior vector $\boldsymbol{\lambda} = [\lambda_1, \dots, \lambda_K]$:
\begin{equation}
    P(\mathbf{p}) = \text{Dir}(\mathbf{p}; \boldsymbol{\lambda}).
\end{equation}

To ensure a valid probability density function without inducing boundary sparsity, we enforce the constraint $\lambda_k \ge 1, \forall k \in [1,K]$. \textit{In this paper, we specifically set the prior vector to $\boldsymbol{\lambda} = \mathbf{1}_K$, where $\mathbf{1}_K$ is a $K$-dimensional vector of ones. The detailed theoretical rationale for this selection is deferred to \hyperref[insight:alpha]{\textbf{Insight 2}} in Section \ref{insights}. To maintain the coherence of the derivation, we retain the general notation $\boldsymbol{\lambda}$ in the subsequent formulations.}

For the generation of the observation $\mathbf{y}$, we assume it is drawn from a Gaussian likelihood distribution centered at the latent variable $\mathbf{p}$:
\begin{equation}
    P(\mathbf{y}|\mathbf{p}) \propto \exp\left(-\frac{1}{2\sigma^2}||\mathbf{y} - \mathbf{p}||_2^2\right).
\label{Gaussian}
\end{equation}

We now derive the two components of the ELBO in Eq. (\ref{ELBO}). For the Expected Negative Log-Likelihood term$-\mathbb{E}_{q_\phi(\mathbf{p}|\mathbf{x})}[\log P(\mathbf{y}|\mathbf{p})]$, with Eq. (\ref{Gaussian}) we have 
\begin{equation}
\begin{aligned}
    &-\mathbb{E}_{q_\phi(\mathbf{p}|\mathbf{x})}[\log P(\mathbf{y}|\mathbf{p})]) \propto \mathbb{E}_{q_\phi} \left[ ||\mathbf{y} - \mathbf{p}||_2^2 \right] \\
    =& \sum_{k=1}^K \left( (y_k - \mathbb{E}_{q_\phi}[p_k])^2 + \text{Var}_{q_\phi}(p_k) \right), \\
    =& \sum_{k=1}^K \left(y_k - \frac{\alpha_k}{S}\right)^2 + \sum_{k=1}^K \frac{\alpha_k(S - \alpha_k)}{S^2(S+1)} \\
    =& \underbrace{\sum_{k=1}^K (y_k - \hat{p}_k)^2}_{\text{Predictive Bias}} + \underbrace{\sum_{k=1}^K \frac{\hat{p}_k(1 - \hat{p}_k)}{S+1}}_{\text{Variance}},
\end{aligned}
\label{LMSE}
\end{equation}
where $S = \sum_{k=1}^K \alpha_k$ and $\hat{p}_k = \alpha_k / S$ for each class $k$. Optimizing Eq. \eqref{LMSE} equals minimizing both the predictive bias and the variance of the class probability distribution simultaneously.

For the Prior Regularization (KL divergence penalty) term, it can be written as $D_{KL}(q_\phi \parallel P) = \mathbb{E}_{q_\phi}[\log q_\phi(\mathbf{p})] - \mathbb{E}_{q_\phi}[\log P(\mathbf{p})]$ by definition. Expanding these expectations and using some properties of the Dirichlet distribution, we obtain the exact analytical form:
\begin{equation}
\begin{aligned}
    D_{KL}(\boldsymbol{\alpha}) =& \log \Gamma(S) - \sum_{k=1}^K \log \Gamma(\alpha_k)+ \sum_{k=1}^K (\alpha_k - \lambda_k)(\psi(\alpha_k) - \psi(S)) \\
    &\quad - \log \Gamma(||\boldsymbol{\lambda}||_1) + \sum_{k=1}^K \log \Gamma(\lambda_k),
\end{aligned}
\label{KL}
\end{equation}
where $\psi(\cdot)$ is the digamma function. The detailed derivation of Eq. \eqref{KL} is provided in Appendix \ref{app:kl_derivation}.

Since the prior parameter vector $\boldsymbol{\lambda}$ is fixed before training, the terms involving only $\boldsymbol{\lambda}$ are strict constants with respect to the network output $\boldsymbol{\alpha}$, yielding zero gradients during backpropagation. Dropping these independent constants, the effective objective for the KL divergence is given by:
\begin{equation}
\begin{aligned}
    \tilde{D}_{KL}(\boldsymbol{\alpha}) =& \log \Gamma(S) - \sum_{k=1}^K \log \Gamma(\alpha_k) + \sum_{k=1}^K (\alpha_k - \lambda_k)(\psi(\alpha_k) - \psi(S)),
\label{LKL}
\end{aligned}
\end{equation}
where $\Gamma(\cdot)$ represents the Gamma function. Integrating the expected reconstruction error in Eq. \eqref{LMSE} and the effective KL divergence in Eq. \eqref{LKL} into Eq. \eqref{ELBO}, the overall variational loss function for our model is formulated as:
\begin{equation}
\begin{aligned}
    \mathcal{L}_{VI-EDL}(\boldsymbol{\alpha}, \mathbf{y}) =& \mathcal{L}_{MSE} (\boldsymbol{\alpha}, \mathbf{y})+ \beta \cdot\tilde{D}_{KL}(\boldsymbol{\alpha})\\
    =&\sum_{k=1}^K (y_k - \hat{p}_k)^2 + \sum_{k=1}^K \frac{\hat{p}_k(1 - \hat{p}_k)}{S+1} +\\
    &\beta\cdot\bigg[\log \Gamma(S) - \sum_{k=1}^K \log \Gamma(\alpha_k) + \sum_{k=1}^K (\alpha_k - \lambda_k)(\psi(\alpha_k) - \psi(S))\bigg].
\end{aligned}
\label{overall}
\end{equation}

\subsection{Cosine Prototype Evidence Layer}

While the aforementioned KL prior regularization term in Eq. \eqref{ELBO} mitigates the issue of excessively large evidence from an optimization perspective, relying solely on loss-driven regularization may still leave room for vulnerability under extreme anomalies. To structurally enforce this constraint and further avoid overconfidence, we introduce a cosine prototype evidence layer in this section to replace the final linear classifier of $f_\phi$, thereby explicitly bounding the maximum amount of generated evidence. Specifically, given an input $\mathbf{x}$, the backbone network projects it into a lower-dimensional embedding space as a feature vector $\mathbf{\hat{x}}$. For each class $k \in \{1, \dots, K\}$, we introduce a learnable class prototype vector $\mathbf{r}_k$, representing the semantic anchor of that class. The evidence $e_k$ for class $k$ is formulated based on the cosine similarity between the feature $\mathbf{\hat{x}}$ and the prototype $\mathbf{r}_k$:
\begin{equation}
    e_k = \text{Softplus} \Big( \gamma \cdot \big( \cos(\mathbf{\hat{x}}, \mathbf{r}_k) - m \big) \Big),
\label{evidence}
\end{equation}
where $\cos(\mathbf{\hat{x}}, \mathbf{r}_k) = \frac{\mathbf{\hat{x}}^\top \mathbf{r}_k}{||\mathbf{\hat{x}}||_2 ||\mathbf{r}_k||_2}$, $\gamma$ is a learnable scale parameter to amplify the bounded cosine similarity, and $m$ is a learnable margin parameter serving as a similarity threshold. 

This cosine prototype layer limits that $e_k\in[0,\text{Softplus}(\gamma\cdot(1-m))]$ to control the growth of the evidence, and it explicitly forces the evidence network to learn compact intra-class representations. Specifically, when an unfamiliar sample is evaluated, its feature $\mathbf{\tilde{x}}$ will naturally exhibit low cosine similarity with all known class prototypes $\mathbf{r}_k$. Bounded by the margin $m$ and the Softplus function, the network will output near-zero evidence for all classes, thereby providing high uncertainty estimates. 


\subsection{Model Training and Prediction}

During the training phase, we employ an annealing warm-up technique. Specifically, we introduce an epoch-dependent annealing factor $\lambda_t = \min(1.0, t/{E_{warmup}})$, where $t$ is the current epoch and $E_{warmup}$ is the number of warm-up epochs. This factor dynamically scales the KL divergence term, allowing the network to focus on fitting the data likelihood in the initial stages and gradually imposing the uncertainty regularization as training stabilizes. 

Consequently, the annealing variational loss function at epoch $t$ is formulated as:
\begin{equation}
    \mathcal{L}_{anneal}^{(t)} = \mathcal{L}_{MSE}(\mathbf{\alpha}, \boldsymbol{y}) + \lambda_t \cdot \beta \cdot \tilde{D}_{KL}(q_\phi(\mathbf{p}|\mathbf{x}) \parallel P(\mathbf{p})).
\label{annealing}
\end{equation}

\begin{algorithm}[t]
\caption{Training process of VI-EDL}
\label{alg:viedl}
\begin{algorithmic}[1]
\State \textbf{Input:} Training dataset $\mathcal{D}=[X,Y]$.
\State \textbf{Parameters:} Hyper-parameter $\beta$ and warm-up epochs $E_{warmup}$.
\State \textbf{Output:} The optimal evidence network $f_{\phi^*}$, class prototype vectors $[\mathbf{r}^*_1, ...,\mathbf{r}^*_K]$, scale parameter $\gamma^*$ and margin parameter $m^*$.
\State Initialize the parameters $\phi$, $[\mathbf{r}_1,...,\mathbf{r}_K]$, $\gamma$ and $m$;
\For{epoch=1, 2, \dots}
    \State Calculate annealing factor $\lambda_t = \min(1.0, t / E_{warmup})$;
    \For{batch=1, 2, \dots}
        \State Sample a mini-batch $(x, y)$ from $\mathcal{D}$;
        \State Extract $\tilde{x}$ for each sample by the backbone network;
        \State Obtain evidence $e_k$ of each class $k$ by Eq. (12);
        \State Calculate $\alpha_k = e_k + 1$, $S = \sum_{k=1}^K \alpha_k$ and $\hat{p}_k = \alpha_k / S$;
        \State Calculate the annealing variational loss by Eq. \eqref{annealing};
        \State Update all parameters via gradient descent;
    \EndFor
\EndFor
\end{algorithmic}
\end{algorithm}

The pseudo-code is detailed in Algorithm \ref{alg:viedl}. After the training process, the optimized network $f_{\phi^*}$ can be deployed for deterministic prediction and epistemic uncertainty quantification simultaneously in a single forward pass. Given a new test instance, the model predicts the categorical distribution by calculating the expected probabilities $\hat{p}_k = \alpha_k / S, S = \sum_{k=1}^K \alpha_k$. The final predicted class label $\hat{y}$ is determined by the maximum expected probability $\hat{y} = \arg\max_{k \in \{1, \dots, K\}} \hat{p}_k.$ Meanwhile, the epistemic uncertainty $u$ of the prediction can be calculated as $u = \frac{\sum_{k=1}^K \lambda_k}{S}$, with $\forall k, \lambda_k=1$.

\section{Theoretical Analysis}
\label{sec:theory}

\renewcommand{\qedsymbol}{$\blacksquare$}

We not only theoretically establish the generalization bound \citep{gen1, gen2} for the proposed model, but also delve into the insights it provides.

\subsection{Basic Assumptions and the Generalization Bound}

We first introduce two essential assumptions regarding the evidence capacity and the input space, which serve as the foundation for our complexity analysis.

\begin{assumption}
\label{limited_evidence}
In Evidential Deep Learning (EDL), the output uncertainty is defined as the ratio of the prior capacity to the total distribution density: $\mu = \frac{||\boldsymbol{\lambda}||_1}{||e||_1 + ||\boldsymbol{\lambda}||_1}\in [0,1]$. We assume a robustness constraint $\mu \ge \mu_{min}$, which mathematically imposes a strict upper bound on the total generated evidence $||e||_1$:
\begin{equation}
    ||e||_1 \le ||\boldsymbol{\lambda}||_1 \left( \frac{1}{\mu_{min}} - 1 \right) \triangleq M.
\end{equation}
\end{assumption}

\begin{assumption}
\label{bounded_input}
The input feature space $\mathcal{X}$ of the neural network is bounded by a maximum radius $R > 0$, such that for any input $x \in \mathcal{X}$, $||x||_2 \le R$.
\end{assumption}

Assumption \ref{limited_evidence} indicates that the total amount of evidence is finite, and Assumption \ref{bounded_input} implies that the input space is bounded, both of which are highly mild and standard assumptions. According to these assumptions, we can obtain the generalization bound for our proposed model.

\begin{theorem}
\label{t4.1}
With probability at least $1 - \delta$, the expected true risk $\mathcal{L}_{true}(\phi)$ of the VI-EDL model is bounded by:

\begin{equation}
\small
\begin{aligned}
    \mathcal{L}_{true}(\phi) \le& \mathcal{L}_{emp}(\phi) + \mathcal{O}\Bigg( \Bigg[ ( 2 + \frac{1}{(K+1)^2} + \frac{2}{K(K+1)} + 2\beta + \frac{\beta}{\min_j \lambda_j} ) ||\boldsymbol{\lambda}||_1 + \beta \bigg]\times  \left( \frac{1}{\mu_{min}} - 1 \right) \\
    &\times \frac{R \sqrt{K} \left( \prod_{l=1}^{L-1} L_{\sigma_l} \right) \left( \prod_{l=1}^L ||W_l||_2 \right)}{\sqrt{n}} \Bigg)+ 3B\sqrt{\frac{\log(2/\delta)}{2n}},
\end{aligned}
\end{equation}
\normalsize
where $\mathcal{L}_{true}$ represents the true generalization risk  (prediction error on unseen samples) and $\mathcal{L}_{emp}$ represents the empirical risk (prediction error on the training set). We assume the evidence network $f_\phi$ is an $L$-layer neural network with weight matrices $W_l$ and $L_{\sigma_l}$-Lipschitz activation functions $\sigma_l$ $(l\in[1,2,...,L])$, and the value of the loss function $\mathcal{L}_{VI-EDL}$ is bounded by $[0, B]\in\mathbb{R}$.
\end{theorem}

The detailed step-by-step proofs are provided in Section \ref{proof4.1} of the Appendix, and the derivation of Theorem \ref{t4.1} proceeds in three progressive steps. First, we establish the global Lipschitz continuity of our variational loss function $\mathcal{L}_{VI-EDL}$ with respect to the generated evidence to bound its gradient. Second, leveraging the Ledoux-Talagrand contraction lemma and the structural unrolling of the neural network, we decouple the loss function and bound the expected Rademacher complexity of the hypothesis space. Finally, we employ McDiarmid's Inequality to tightly concentrate the empirical Rademacher complexity around its expectation, bridging the gap between empirical and true risks. 

\subsection{Theoretical Insights of the Generalization Bound} 
\label{insights}

Theorem \ref{t4.1} demonstrates that the true generalization risk $\mathcal{L}_{true}$ is rigorously bounded by the empirical risk $\mathcal{L}_{emp}$ and an additional methematical gap. This bound formally guarantees the generalization capability of the proposed VI-EDL framework,  ensuring that the model does not merely memorize training data but learns transferable patterns from the underlying class probability distribution. Moreover, it encapsulates the interplay between the data, the network architecture, and the uncertainty quantification mechanism, yielding several insights into the learning dynamics of the proposed model:

\begin{itemize}
    \item \textbf{1. Sample Size ($n$) and Dimensionality of the Label Space ($K$):} The generalization bound explicitly decreases at a rate of $\mathcal{O}(1/\sqrt{n})$. As the number of training samples $n$ increases, the bound becomes tighter, which naturally guaranties that the empirical error converges to the true expected risk. The number of classes $K$ plays a dual role. If $K=0$, the bound naturally tends to infinity, indicating that the learning problem is unsolvable. In contrast, as $K \to \infty$, the uncertainty is maximally compressed by the known labels, pushing the generalization bound to its minimum valid state. Intuitively, if a task possesses an infinitely comprehensive label space that encompasses all possible categories, any given instance would be completely interpretable by the known semantics. Under such ideal conditions, the uncertainty naturally vanishes, as the very concept of "unknowns" ceases to exist.

    \label{insight:alpha}
    \item \textbf{2. Optimality of the Uniform Prior ($\boldsymbol{\lambda}$):} The generalization bound is a strictly monotonically increasing linear function of $||\boldsymbol{\lambda}||_1$. To achieve the tightest possible bound, one must strictly minimize $||\boldsymbol{\lambda}||_1$. Given the validity constraint $\forall i ,\lambda_i \ge 1$ (required to prevent singular Dirichlet prior densities), the global minimum is achieved when $\boldsymbol{\lambda} = \mathbf{1}_K$, where $\mathbf{1}_K$ is a $K$-dimensional vector of ones. This mathematically proves that initializing with a uniform Dirichlet prior is not merely an empirical heuristic, but the theoretically optimal choice.

    \label{insight:uncertainty}
    \item \textbf{3. Trade-off via Minimum Uncertainty ($\mu_{min}$):} The minimum uncertainty $\mu_{min}$ acts as an uncertainty amplifier. If $\mu_{min} \to 0$, the amplifier tends to infinity. In this case, evidence provided by the model tends to infinity, indicating severe overconfidence. However, while forcing $\mu_{min} = 1$ makes the bound vanish, it completely destroys the predictive capability of the model. This model just output no evidence and total uncertainty for any instance, inflating the empirical error $\mathcal{L}_{emp}(\phi)$. Therefore, $\mu_{min}$ must be a balanced value, ideally learned dynamically by the model itself.

    \label{insight:complexity}
    \item \textbf{4. Network Capacity and Feature Complexity ($R, W_l, \sigma_l$):} The bound is directly proportional to the maximum feature radius $R$, the spectral norms of the weight matrices $||W_l||_2$, and the Lipschitz constants of the activation functions $L_{\sigma_l}$. Larger values imply a more complex hypothesis family, which increases the difficulty of fitting and degrades generalization. This also highlights the necessity of architectural regularization in evidential networks.
\end{itemize}

\section{Experiments}
\label{sec:experiments}

\subsection{Experimental Setup}

\textbf{Baselines.} Following standard evaluation protocols, we primarily focus on comparing our VI-EDL with other EDL methods, including the traditional EDL \citep{EDL}, $\mathcal{I}$-EDL \citep{IEDL}, Re-EDL \citep{REEDL}, and F-EDL \citep{FEDL}. Additionally, we present the results of other uncertainty quantification methods, including the probabilistic posterior network NatPN \citep{NATPN}, the deterministic baseline Softmax (CE), and the Deep Ensemble \citep{DE} method using 5 model instances ($M=5$).

\textbf{Datasets.} We select CIFAR-10 \citep{CIFAR}, SVHN \citep{SVHN}, Flowers \citep{Flowers}, CIFAR-100 \citep{CIFAR}, and MedMNIST \citep{MedMNIST} (including Blood, Path, Tissue, and OrganMNIST) for evaluation. In detail, We utilize four widely used natural image datasets in our experiments:
\begin{itemize}
    \item \textbf{CIFAR-10 and CIFAR-100} \citep{CIFAR}: Both datasets consist of 60,000 color images with a resolution of $32 \times 32$ pixels. CIFAR-10 is categorized into 10 classes of natural objects and animals (with 6,000 images per class), while CIFAR-100 contains 100 fine-grained classes (with 600 images per class).
    \item \textbf{SVHN} \citep{SVHN}: This dataset contains over 600,000 color images with a resolution of $32 \times 32$ pixels. The images depict real-world printed digits (0 to 9) cropped from house number plates in Google Street View.
    \item \textbf{Flowers} \citep{Flowers}: This dataset consists of 8,189 high-resolution images of flowers commonly occurring in the United Kingdom. The images are categorized into 102 distinct flower species.
\end{itemize}

Moreover, we also evaluate our model on four subsets from the MedMNIST v2 benchmark \citep{MedMNIST}, which provides standardized 2D biomedical images resized to $28 \times 28$ pixels:
\begin{itemize}
    \item \textbf{BloodMNIST:} Comprises 17,092 optical microscope images depicting normal peripheral blood cells. The images are categorized into 8 distinct cell types.
    \item \textbf{PathMNIST:} Contains 107,180 histological images of colorectal cancer tissues. The images are classified into 9 different tissue types.
    \item \textbf{TissueMNIST:} Consists of 236,386 monochromatic images showing human kidney cortex cells. These cells are segmented from wide-field microscopy images and classified into 8 categories.
    \item \textbf{OrganMNIST:} Derived from 3D abdominal CT scans that are processed into 2D bounding-box crops. It is categorized into 11 distinct abdominal organs. Taking the Axial view subset as an example, it contains 58,850 images.
\end{itemize}



\textbf{Implementation.} SVHN \citep{SVHN}, Flowers \citep{Flowers} and CIFAR-100 \citep{CIFAR} are utilized as OOD data for CIFAR-10 \citep{CIFAR}, while PathMNIST \citep{MedMNIST}, TissueMNIST \citep{MedMNIST} and OrganMNIST \citep{MedMNIST} are used for BloodMNIST \citep{MedMNIST}. Resnet18 serves as the backbone network for both CIFAR-10 and BloodMNIST. The Adam optimizer is employed with a learning rate of $1\times10^{-3}$ for CIFAR-10 and BloodMNIST. The hyper-parameter $\beta$ is set to 0.5 and 0.1 for CIFAR-10 and BloodMNIST,which is selected from the range [0.1:0.1:1.0] on the validation set. The batch size is set to 256, and the warm-up and maximum epoch is set to 20 and 30 for CIFAR-10 and BloodMNIST. Reported results are averaged over 5 runs. 


\begin{table}[!tbp]
\centering
\caption{OOD detection results on CIFAR-10, with mean and standard deviation reported over five runs. The best results are highlighted in \textbf{bold}. All results are shown in percentiles (\%).}
\label{tab:cifar10_full_std}
\resizebox{\textwidth}{!}{%
\begin{tabular}{l | c | cccc | cccc}
\toprule
\multirow{2}{*}{\textbf{Method}} & \multirow{2}{*}{\textbf{ID ACC ($\uparrow$)}} & \multicolumn{4}{c|}{\textbf{OOD AUROC ($\uparrow$)}} & \multicolumn{4}{c}{\textbf{OOD FPR95 ($\downarrow$)}} \\ \cmidrule(lr){3-6} \cmidrule(lr){7-10}
 & & SVHN & Flowers & CIFAR-100 & Avg. & SVHN & Flowers & CIFAR-100 & Avg. \\ \midrule
Softmax (CE) & 82.85 $\pm$ 0.15 & 83.82 $\pm$ 0.12 & 80.94 $\pm$ 0.23 & 72.71 $\pm$ 0.45 & 79.16 & 79.31 $\pm$ 0.21 & 78.01 $\pm$ 0.34 & 87.27 $\pm$ 0.41 & 81.53 \\
Deep Ensembles ($M=5$) & 91.12 $\pm$ 0.11 & 87.32 $\pm$ 0.08 & 84.14 $\pm$ 0.15 & 84.27 $\pm$ 0.11 & 85.24 & 53.86 $\pm$ 0.12 & 70.90 $\pm$ 0.21 & 60.01 $\pm$ 0.18 & 61.59 \\
NatPN (ICLR 2022) & 86.45 $\pm$ 0.21 & 86.19 $\pm$ 0.11 & 85.11 $\pm$ 0.14 & 85.01 $\pm$ 0.19 & 85.44 & 71.69 $\pm$ 0.25 & 68.53 $\pm$ 0.31 & 65.68 $\pm$ 0.28 & 68.63 \\
EDL (NeurIPS 2018) & 82.03 $\pm$ 0.35 & 80.38 $\pm$ 0.15 & 82.24 $\pm$ 0.21 & 82.57 $\pm$ 0.18 & 81.73 & 90.47 $\pm$ 0.35 & 73.62 $\pm$ 0.41 & 70.59 $\pm$ 0.39 & 78.23 \\
$\mathcal{I}$-EDL (ICML 2023) & 87.33 $\pm$ 0.24 & 85.22 $\pm$ 0.12 & 84.78 $\pm$ 0.19 & 84.48 $\pm$ 0.22 & 84.83 & 69.12 $\pm$ 0.28 & 59.48 $\pm$ 0.33 & 62.98 $\pm$ 0.35 & 63.86 \\
Re-EDL (TPAMI 2025) & 90.10 $\pm$ 0.09 & 90.92 $\pm$ 0.09 & 86.52 $\pm$ 0.11 & 87.40 $\pm$ 0.08 & 89.61 & 57.70 $\pm$ 0.15 & 66.68 $\pm$ 0.19 & \textbf{57.74} $\pm$ 0.12 & 60.71 \\
F-EDL (NeurIPS 2025) & 89.27 $\pm$ 0.18 & 82.22 $\pm$ 0.14 & 87.25 $\pm$ 0.15 & 87.28 $\pm$ 0.21 & 85.58 & 45.35 $\pm$ 0.18 & 57.10 $\pm$ 0.22 & 61.42 $\pm$ 0.25 & 54.62 \\
\midrule
\textbf{VI-EDL (Ours)} & \textbf{92.32} $\pm$ 0.10 & \textbf{93.62} $\pm$ 0.05 & \textbf{87.86} $\pm$ 0.08 & \textbf{87.55} $\pm$ 0.12 & \textbf{91.33} & \textbf{19.11} $\pm$ 0.09 & \textbf{56.12} $\pm$ 0.12 & 65.61 $\pm$ 0.15 & \textbf{46.95} \\
\bottomrule
\end{tabular}%
}
\label{tab:cifar10_ood}

\end{table}

\begin{table}[!tbp]
\centering
\caption{OOD detection results on BloodMNIST, with mean and standard deviation reported over five runs. The best results are highlighted in \textbf{bold}. All results are shown in percentiles (\%).}
\label{tab:medmnist_full_std}
\resizebox{\textwidth}{!}{%
\begin{tabular}{l | c | cccc | cccc}
\toprule
\multirow{2}{*}{\textbf{Method}} & \multirow{2}{*}{\textbf{ID ACC ($\uparrow$)}} & \multicolumn{4}{c|}{\textbf{OOD AUROC ($\uparrow$)}} & \multicolumn{4}{c}{\textbf{OOD FPR95 ($\downarrow$)}} \\ \cmidrule(lr){3-6} \cmidrule(lr){7-10}
 & & PathMNIST & TissueMNIST & OrganMNIST & Avg. & PathMNIST & TissueMNIST & OrganMNIST & Avg. \\ \midrule
Softmax (CE) & 80.18 $\pm$ 0.22 & 45.05 $\pm$ 0.45 & 55.53 $\pm$ 0.48 & 30.39 $\pm$ 0.67 & 43.66 & 98.05 $\pm$ 0.12 & 92.69 $\pm$ 0.09 & 98.02 $\pm$ 0.15 & 96.25 \\
Deep Ensembles ($M=5$) & \textbf{90.21} $\pm$ 0.15 & 43.95 $\pm$ 0.32 & 59.66 $\pm$ 0.54 & 51.34 $\pm$ 0.41 & 51.65 & 98.48 $\pm$ 0.18 & 92.50 $\pm$ 0.42 & 90.73 $\pm$ 0.33 & 93.90 \\
NatPN (ICLR 2022) & 89.21 $\pm$ 0.28 & 62.82 $\pm$ 0.19 & 63.25 $\pm$ 0.33 & 68.44 $\pm$ 0.22 & 64.84 & 98.87 $\pm$ 0.11 & 96.15 $\pm$ 0.25 & 73.78 $\pm$ 0.37 & 89.60 \\
EDL (NeurIPS 2018) & 82.72 $\pm$ 0.41 & 31.30 $\pm$ 0.55 & 43.01 $\pm$ 0.09 & 15.33 $\pm$ 0.48 & 29.88 & 99.58 $\pm$ 0.08 & 92.22 $\pm$ 0.67 & 99.51 $\pm$ 0.10 & 97.10 \\
$\mathcal{I}$-EDL (ICML 2023) & 89.65 $\pm$ 0.31 & 61.32 $\pm$ 0.21 & 53.23 $\pm$ 0.32 & 37.77 $\pm$ 0.39 & 50.77 & 88.62 $\pm$ 0.44 & 98.74 $\pm$ 0.15 & 92.05 $\pm$ 0.28 & 93.14 \\
Re-EDL (TPAMI 2025) & 86.00 $\pm$ 0.18 & 57.33 $\pm$ 0.42 & 59.94 $\pm$ 0.60 & 52.53 $\pm$ 0.31 & 56.60 & 94.00 $\pm$ 0.26 & 90.98 $\pm$ 0.10 & 72.52 $\pm$ 0.41 & 85.83 \\
F-EDL (NeurIPS 2025) & 86.64 $\pm$ 0.25 & 41.01 $\pm$ 0.38 & 53.30 $\pm$ 0.36 & 52.60 $\pm$ 0.44 & 48.97 & 99.54 $\pm$ 0.09 & 93.81 $\pm$ 0.22 & 93.76 $\pm$ 0.32 & 95.70 \\
\midrule
\textbf{VI-EDL (Ours)} & 89.71 $\pm$ 0.12 & \textbf{72.01} $\pm$ 0.14 & \textbf{76.25} $\pm$ 0.22 & \textbf{77.48} $\pm$ 0.18 & \textbf{75.25} & \textbf{33.43} $\pm$ 0.21 & \textbf{25.34} $\pm$ 0.35 & \textbf{53.43} $\pm$ 0.29 & \textbf{37.40} \\
\bottomrule
\end{tabular}%
}
\label{tab:medmnist_ood}

\end{table}

\begin{figure}[!tbp]
    \centering
    \begin{minipage}[b]{0.32\textwidth}
        \centering
        \includegraphics[width=\linewidth]{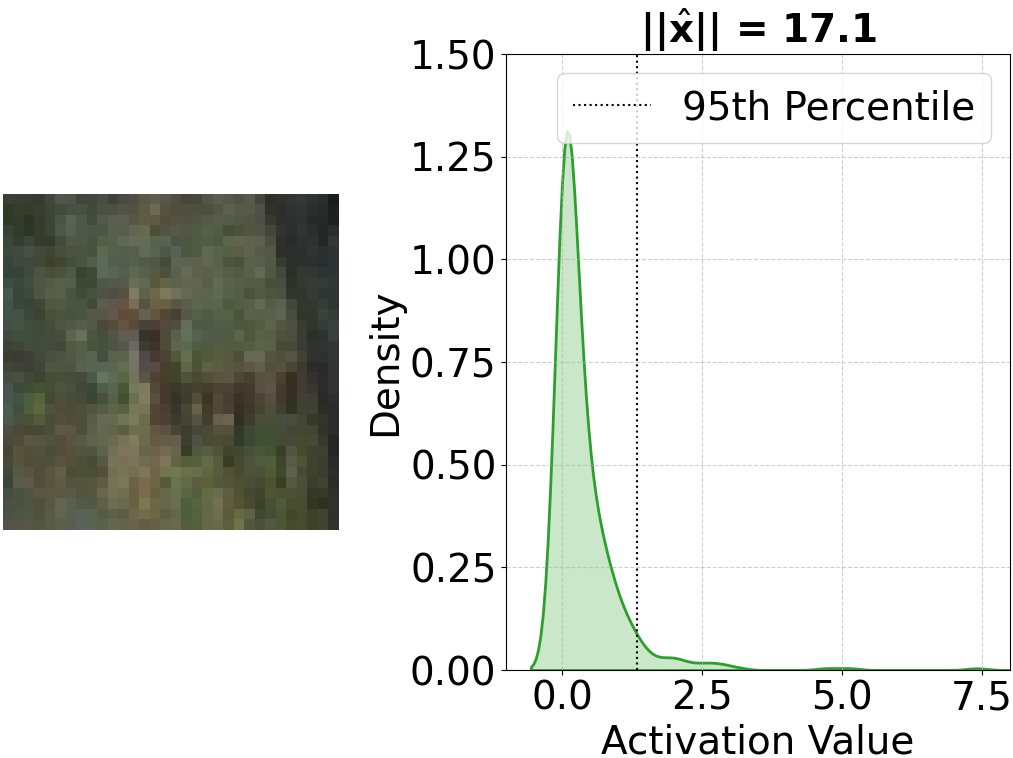}
        
        \vspace{0.1cm} 
        \small (a) CIFAR-10
    \end{minipage}
    \hfill 
    \begin{minipage}[b]{0.32\textwidth}
        \centering
        \includegraphics[width=\linewidth]{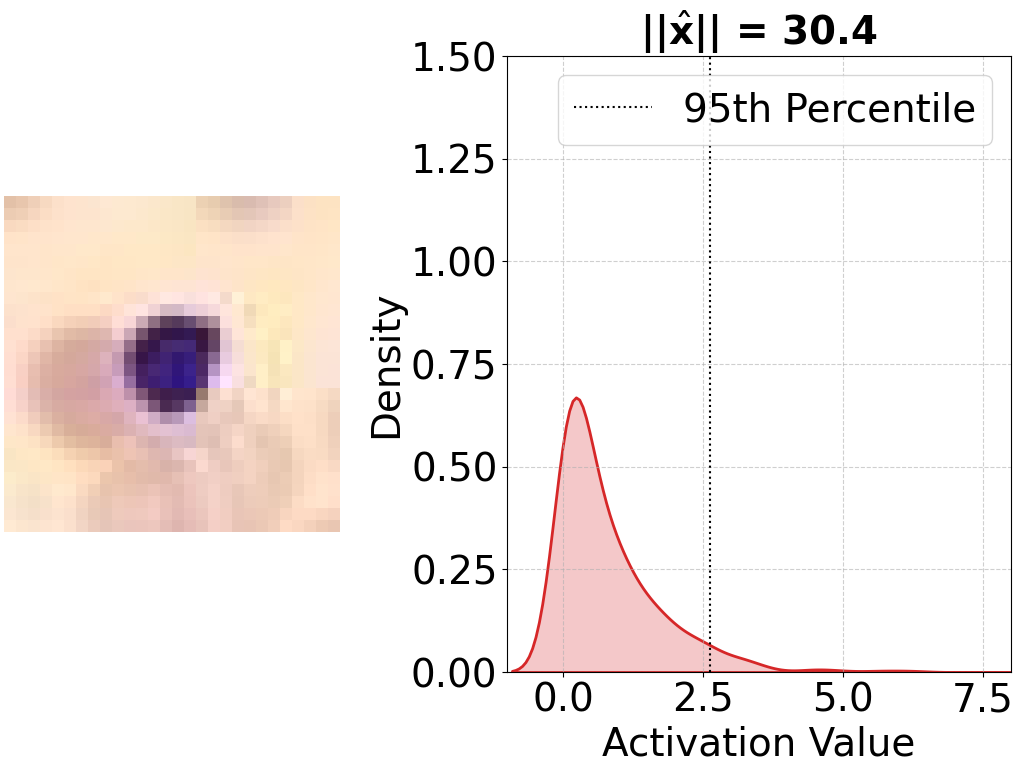}
        
        \vspace{0.1cm}
        \small (b) BloodMNIST
    \end{minipage}
    \hfill
    \begin{minipage}[b]{0.32\textwidth}
        \centering
        \includegraphics[width=\linewidth]{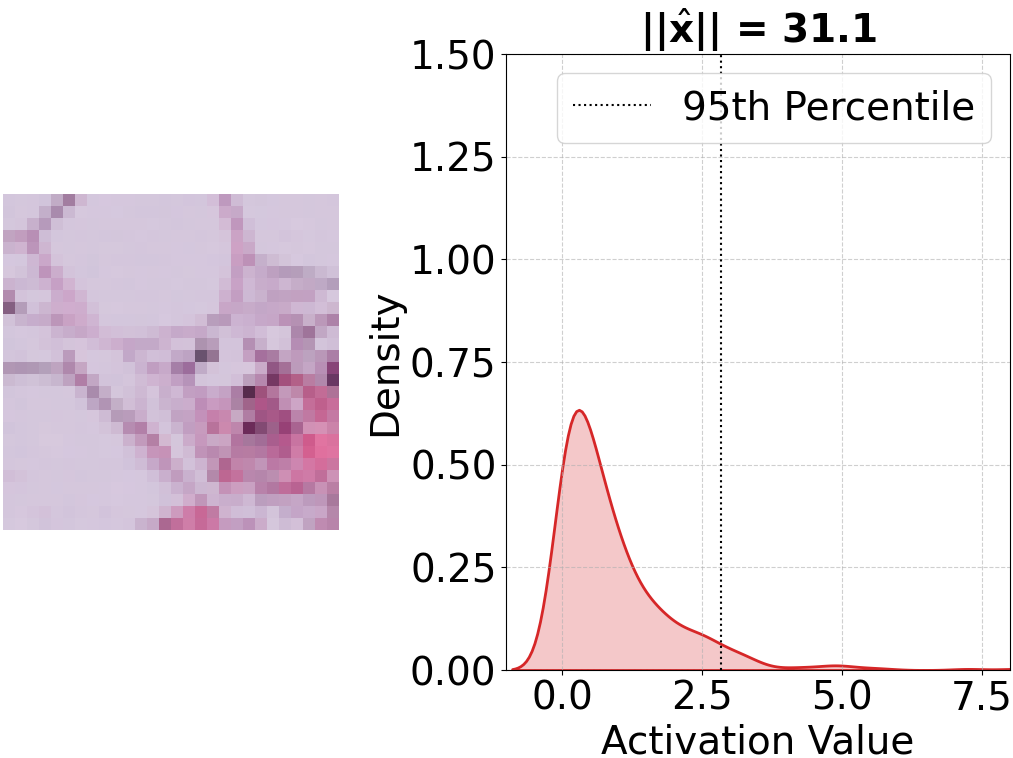}
        
        \vspace{0.1cm}
        \small (c) PathMNIST
    \end{minipage}
    
    \caption{Comparison of feature magnitude $\|\hat{\mathbf{x}}\|$.The left panel of each subfigure is the original image. The right panel shows the kernel density curve, which illustrates the density distribution of the 512-dimensional activation values extracted by the convolutional kernels for the given image. The dotted line indicates the 95th percentile, and the area under this curve equals 1. \textbf{(a)} A natural image from CIFAR-10 serves as a control, showing concentrated activations. \textbf{(b)} \& \textbf{(c)} Two medical samples from BloodMNIST and PathMNIST exhibiting a broader density distribution, triggering larger feature magnitude $\|\hat{\mathbf{x}}\|$.}
    \label{fig:magnitude_comparison}
    
\end{figure}

\subsection{Out-of-Distribution and Noise Robustness Evaluation}

We conduct extensive experiments focusing on Out-of-Distribution detection and robustness against data noise. The results are reported in classification accuracy on in-distribution (ID) samples (ID ACC ($\uparrow$)) and the capability to distinguish unknown from known samples (OOD AUROC ($\uparrow$) and OOD FPR95 ($\downarrow$)).

\textbf{OOD Detection.}
We first train the models on the CIFAR-10 and BloodMNIST dataset and run OOD detection test on three other datasets, with the results summarized in Table \ref{tab:cifar10_ood} and \ref{tab:medmnist_ood}. These results demonstrate that our proposed VI-EDL outperforms existing baselines across most uncertainty metrics. Notably, EDL baselines often suffer from a noticeable degradation in ID ACC due to the optimization conflict introduced by heuristic KL penalties. In contrast, our method successfully maintains a highly competitive ID ACC while simultaneously achieving most of the highest AUROC and the lowest FPR95. This indicates that the cosine prototype layer and KL penalty on all classes effectively neutralizes magnitude-induced overconfidence when encountering OOD samples. Furthermore, it is noteworthy that the overall performance in Table \ref{tab:medmnist_ood} is generally inferior to that in Table \ref{tab:cifar10_ood}. We hypothesize that this discrepancy arises because the maximum feature magnitude of the BloodMNIST dataset, after being mapped through the CNN backbone, is larger than that of CIFAR-10. 

To verify this, we extract representative samples from the MedMNIST benchmark (specifically, BloodMNIST and PathMNIST) and compare their pre-logits feature magnitudes $\|\hat{\mathbf{x}}\|$ mapped by a standard ResNet-18 backbone against a natural image from CIFAR-10. The results are illustrated in Fig. \ref{fig:magnitude_comparison}. As shown in Fig. \ref{fig:magnitude_comparison} \textbf{(a)}, the feature activations of the natural image in CIFAR-10 are well-concentrated within a narrow range. In contrast, the two medical images in Fig. \ref{fig:magnitude_comparison} \textbf{(b)} \& \textbf{(c)} exhibit relatively heavy-tail density distributions. This phenomenon likely occurs because CNNs are highly sensitive to domain-specific interferences such as cell staining variations and acquisition artifacts. Coupled with the fact that biological tissues possess inherently more complex structures than natural images, these factors over-stimulate the convolutional filters, causing the network to output abnormally high feature values. As indicated by \hyperref[insight:complexity]{\textbf{Insight 4}} in Section \ref{insights}, a larger feature radius yields a looser generalization bound, which consequently leads to a degradation in the method's performance.

\begin{figure}[!tbp]
    \centering
    
    \includegraphics[width=0.9\textwidth]{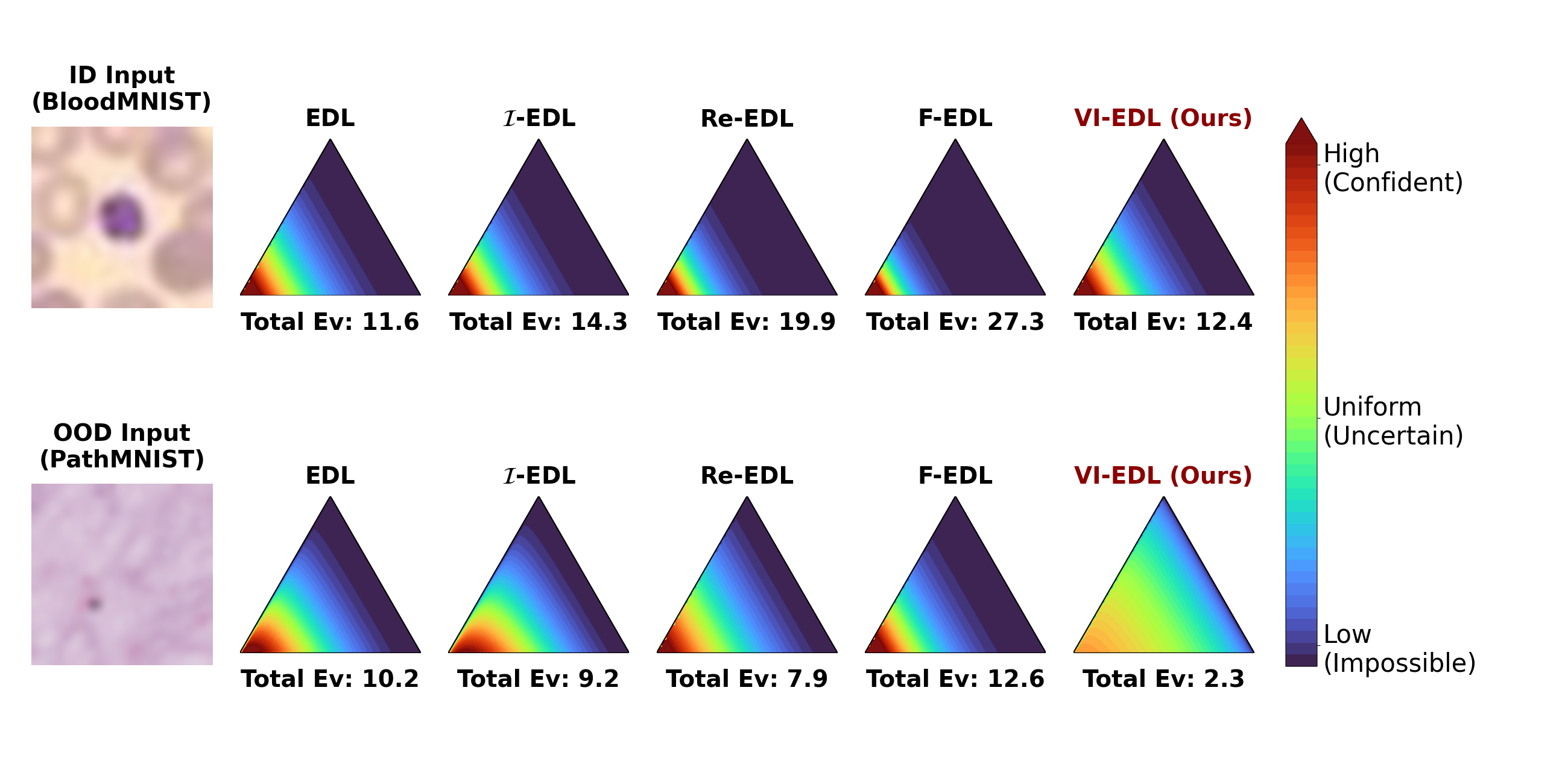} 
    
    \caption{Visualization of the predicted class probability distributions for EDL baselines and our method in the BloodMNIST and PathMNIST dataset. Each vertex of the triangles represents a specific class (with top-3 evidence), and total evidence of all the classes (Total Ev) is presented below triangles. As is shown in this figure, our method generates \textbf{high evidence} for the ID input and \textbf{lowest evidence} for the OOD input, indicating its effectiveness in OOD detection among these datasets. }
    \label{fig:simplex_medmnist}
\end{figure}

\textbf{Visualization Analysis.}
We visualize the generated evidence of each EDL method in Table \ref{tab:medmnist_ood} to further analyze the performance of our method. Since the 8-class BloodMNIST (9-class PathMNIST) dataset naturally produces a 7-dimensional (8-dimensional) probability simplex, we dynamically project the predictions onto a 2-simplex (triangle) by extracting the top-3 categories with the highest evidence for each instance. Moreover, we construct the hardest possible OOD samples by pairing a randomly selected ID image with its closest OOD counterpart. This is retrieved by minimizing the Mean Squared Error (MSE) between the fixed ID sample and all available OOD samples. As illustrated in Figure \ref{fig:simplex_medmnist}, for confident ID samples, the evidence is sharply concentrated at the corresponding category vertex for all methods, showing high confidence. However, for OOD inputs, other EDL baselines still generate relatively high evidence. In contrast, the evidence of our method is suppressed across all dimensions, indicating high uncertainty.

A plausible underlying reason of the failure for these EDL baselines may lie in the feature extraction stage. Given a raw input image $\mathbf{x}$, the backbone network maps it to a deep feature representation $\hat{\mathbf{x}}$. In medical domains, raw images are inevitably subject to physical acquisition variations. When processed by a deep CNN backbone, these non-semantic perturbations may over-stimulate convolutional filters, inducing abnormally large feature magnitudes $\|\hat{\mathbf{x}}\|$ (See Fig. \ref{fig:magnitude_comparison}). Conventional linear evidence layers conflate this raw intensity with semantic relevance, generating excessive high evidence. However, by applying a KL divergence penalty and cosine prototype layer, our method decouples semantic alignment from feature intensity and suppresses evidence of all classes, ensuring that abnormally large $||\hat{\mathbf{x}}||$ fails to generate a high evidence. 

\begin{table}[!tbp]
\centering
\caption{Noise detection results on CIFAR-10, with mean and standard deviation reported over five runs. A higher value of $\sigma$ indicates a more intense additive Gaussian noise perturbation on the original data. The best results are highlighted in \textbf{bold}. All results are shown in percentiles (\%).}
\resizebox{\textwidth}{!}{%
\begin{tabular}{l |  cccc | cccc}
\toprule
\multirow{2}{*}{\textbf{Method}}  & \multicolumn{4}{c|}{\textbf{OOD AUROC ($\uparrow$)}} & \multicolumn{4}{c}{\textbf{OOD FPR95 ($\downarrow$)}} \\ \cmidrule(lr){2-5} \cmidrule(lr){6-9}
& $\sigma=0.05$ & $\sigma=0.10$ & $\sigma=0.20$ & Avg. & $\sigma=0.05$ & $\sigma=0.10$ & $\sigma=0.20$ & Avg. \\ \midrule
Softmax (CE)  & 51.24 $\pm$ 0.35 & 53.68 $\pm$ 1.12 & 55.92 $\pm$ 2.15 & 53.61 & 96.45 $\pm$ 0.28 & 95.12 $\pm$ 0.95 & 93.84 $\pm$ 2.41 & 95.14 \\
Deep Ensembles ($M=5$) & 57.85 $\pm$ 0.42 & 71.34 $\pm$ 1.25 & 82.46 $\pm$ 2.34 & 70.55 & 86.30 $\pm$ 0.38 & 68.55 $\pm$ 1.08 & 51.20 $\pm$ 2.12 & 68.68 \\
NatPN (ICLR 2022) & 55.42 $\pm$ 0.48 & 66.89 $\pm$ 1.30 & 76.53 $\pm$ 1.95 & 66.28 & 89.75 $\pm$ 0.45 & 75.40 $\pm$ 1.42 & 64.18 $\pm$ 2.65 & 76.44 \\
EDL (NeurIPS 2018)  & 53.15 $\pm$ 0.51 & 61.20 $\pm$ 1.18 & 68.75 $\pm$ 2.50 & 61.03 & 92.60 $\pm$ 0.40 & 83.45 $\pm$ 1.35 & 77.30 $\pm$ 2.75 & 84.45 \\
$\mathcal{I}$-EDL (ICML 2023)  & 54.88 $\pm$ 0.39 & 64.55 $\pm$ 1.45 & 72.18 $\pm$ 2.20 & 63.87 & 90.15 $\pm$ 0.35 & 79.80 $\pm$ 1.20 & 72.55 $\pm$ 2.45 & 80.83 \\
Re-EDL (TPAMI 2025) & 58.60 $\pm$ 0.45 & 73.25 $\pm$ 1.05 & 86.40 $\pm$ 2.18 & 72.75 & 85.25 $\pm$ 0.42 & 65.10 $\pm$ 1.15 & 45.80 $\pm$ 2.05 & 65.38 \\
F-EDL (NeurIPS 2025)  & 59.35 $\pm$ 0.32 & 75.80 $\pm$ 1.38 & 89.15 $\pm$ 2.60 & 74.77 & 83.90 $\pm$ 0.36 & 62.45 $\pm$ 1.28 & 41.30 $\pm$ 2.35 & 62.55 \\
\midrule
\textbf{VI-EDL (Ours)}& \textbf{62.16} $\pm$ 0.55 & \textbf{81.52} $\pm$ 1.34 & \textbf{98.59} $\pm$ 2.62 & \textbf{80.76} & \textbf{81.21} $\pm$ 0.49 & \textbf{43.64} $\pm$ 1.41 & \textbf{24.06} $\pm$ 1.87 & \textbf{49.64} \\
\bottomrule
\end{tabular}%
}
\label{tab:cifar10_noise}
\end{table}

\textbf{Robustness Against Data Noise.}
We also assess the robustness of the proposed model against varying degrees of data noise in CIFAR-10, as detailed in Table \ref{tab:cifar10_noise}. Our method consistently outperforms all baseline approaches across different noise intensities, effectively distinguishing heavily corrupted inputs from clean data, highlighting the model's exceptional noise robustness. The results corroborate that the integration of a global KL divergence penalty alongside the cosine prototype mechanism enforces strict regulation over the evidence generation process, thereby successfully preventing the model from hallucinating high-confidence predictions when encountering various noisy samples.

\subsection{Autonomous Driving Scenario.} 
Moreover, we run our VI-EDL and two other baseline EDL methods (Re-EDL and F-EDL) in the real-world autonomous driving scenario on the BDD100K dataset, which is a large-scale and highly diverse driving video dataset including standard driving scenes. To rigorously assess the robustness of these methods under anomalous conditions, we construct an anomalous set by sampling 1,000 clean images from the dataset and applying synthetic perturbations. As illustrated in Fig. \ref{fig:dataset_showcase}, we generate 500 images simulating rainy conditions and 500 simulating foggy conditions. The three methods are then tasked with identifying this newly constructed anomalous dataset. The results are detailed in Table \ref{tab_UNC}. Notably, VI-EDL exhibits the most significant discrepancy in uncertainty estimates between normal and anomalous driving conditions. Specifically, the uncertainty difference (Unc. Diff) of our proposed model reaches 0.1380, which is about 4.63 times higher than that of Re-EDL and 2.13 times that of F-EDL. This uncertainty gap demonstrates a superior capability of VI-EDL to reliably differentiate between normal and anomalous driving scenarios. 

\begin{figure}[!tbp] 
    \centering
    \begin{minipage}[b]{0.32\textwidth}
        \centering
        \includegraphics[width=\linewidth]{ID_d.png}
        
        \small (a) Normal: Clear
    \end{minipage}
    \hfill
    \begin{minipage}[b]{0.32\textwidth}
        \centering
        \includegraphics[width=\linewidth]{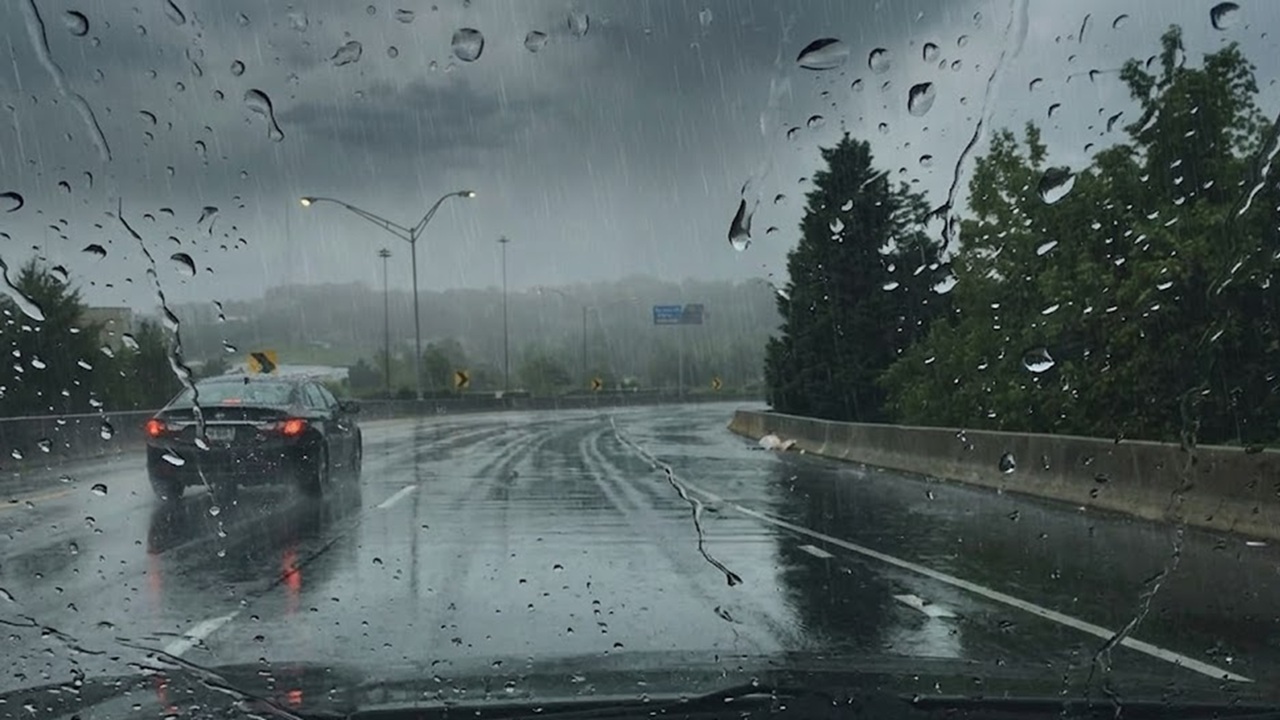} 
        
        \small (b) Anomalous: Rainy
    \end{minipage}
    \hfill
    \begin{minipage}[b]{0.32\textwidth}
        \centering
        \includegraphics[width=\linewidth]{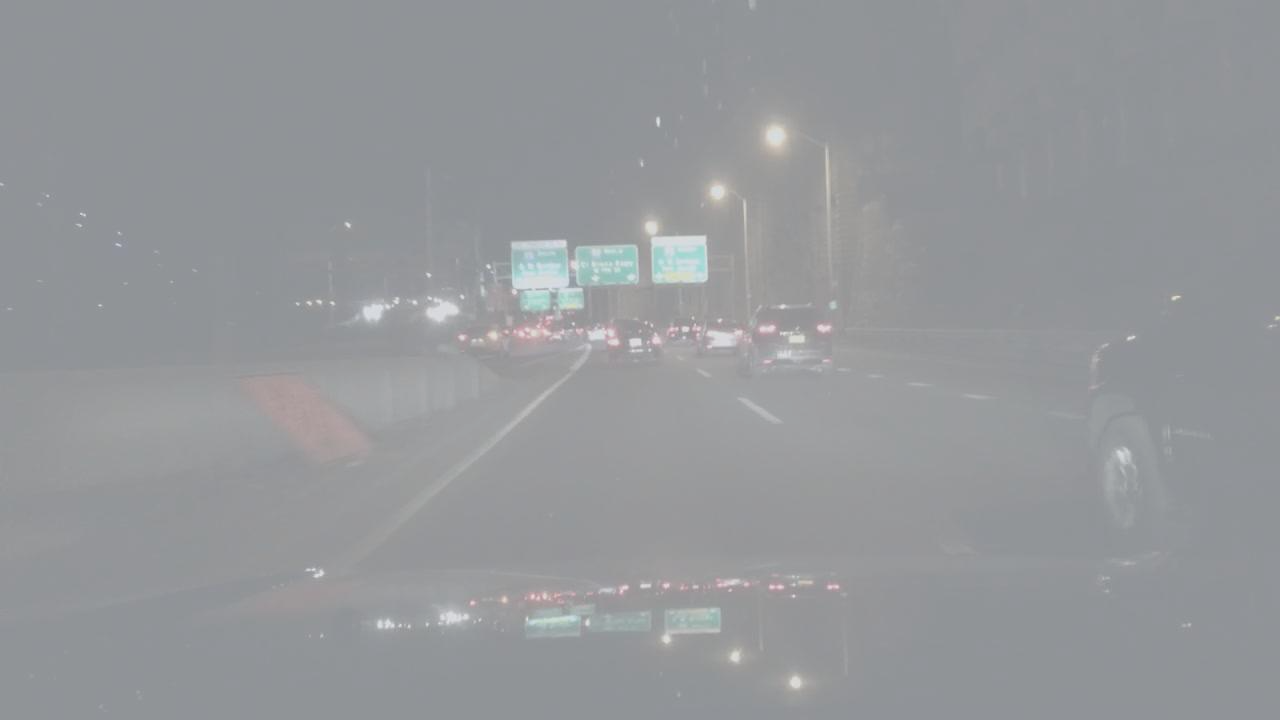}
        
        \small (c) Anomalous: Foggy
    \end{minipage}
    
    \caption{Comparison of normal and anomalous weather in autonomous driving. \textbf{(a)} Normal: a typical driving scene with optimal lighting. \textbf{(b)} Anomalous: a rainy night where dense water droplets distort features. \textbf{(c)} Anomalous: a foggy environment with extreme atmospheric scattering.}
    \label{fig:dataset_showcase}
\end{figure}

We attribute this superior performance to the structural constraints of VI-EDL. These weather conditions typically introduce complex visual distortions that yield aberrant feature magnitudes, and our cosine prototype layer and KL penalty effectively neutralizing magnitude-induced overconfidence. This also proves that our method is highly practical: when the estimated uncertainty exceeds a predefined threshold, the autonomous driving system can proactively issue an alert, prompting human driver intervention and thereby enhancing overall driving safety.

\begin{table}[!tbp]
\centering

\begin{minipage}[t]{0.47\textwidth}
\vspace{0pt}
\centering
\small
\resizebox{\linewidth}{!}{
\begin{tabular}{lccc}
\toprule
Method & Normal Unc. & Anomalous Unc. & Unc. Diff. ($\uparrow$) \\
\midrule
Re-EDL & 0.2333 & 0.2631 & 0.0298 \\
F-EDL & 0.2500 & 0.3148 & 0.0648 \\
VI-EDL (Ours) & \textbf{0.4126} & \textbf{0.5806} & \textbf{0.1380} \\
\bottomrule
\end{tabular}
}
\caption{Comparison of uncertainty across different methods on normal and anomalous samples in the BDD100K dataset. A larger uncertainty difference indicates a better ability to distinguish OOD samples from ID samples.}
\label{tab_UNC}
\end{minipage}
\hfill
\begin{minipage}[t]{0.49\textwidth}
\vspace{0pt}
\centering
\small
\resizebox{\linewidth}{!}{
\begin{tabular}{cc|c|ccc}
\toprule
KL & Cosine & ID ACC & SVHN & FLOWERS & CIFAR-100 \\
\midrule
$\times$ & $\times$ & 66.93 & 33.32 & 65.10 & 61.84 \\
$\times$ & $\checkmark$ & 80.32 & 81.51 & 78.18 & 78.81 \\
$\checkmark$ & $\times$ & 73.87 & 82.52 & 73.36 & 78.37 \\
$\checkmark$ & $\checkmark$ & \textbf{92.32} & \textbf{93.62} & \textbf{87.86} & \textbf{87.55} \\
\bottomrule
\end{tabular}
}
\caption{Ablation study on KL divergence and cosine prototype layer. CIFAR-10 is used as the ID dataset, while SVHN, Flowers, and CIFAR-100 are used as OOD datasets.}
\label{tab:ablation}
\end{minipage}

\end{table}

\begin{figure}[!tbp] 
    \centering
    \begin{minipage}[b]{0.48\textwidth}
        \centering
        \includegraphics[width=\linewidth]{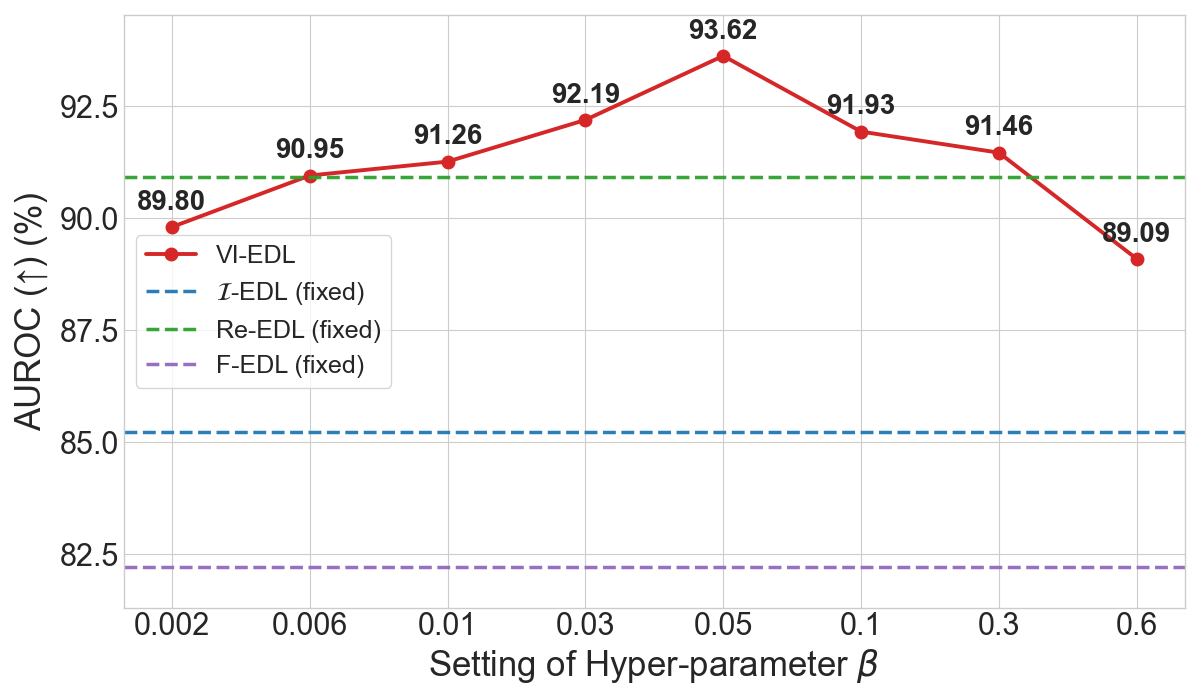}
        
        \small (a) AUROC ($\uparrow$)
    \end{minipage}
    \hfill
    \begin{minipage}[b]{0.48\textwidth}
        \centering
        \includegraphics[width=\linewidth]{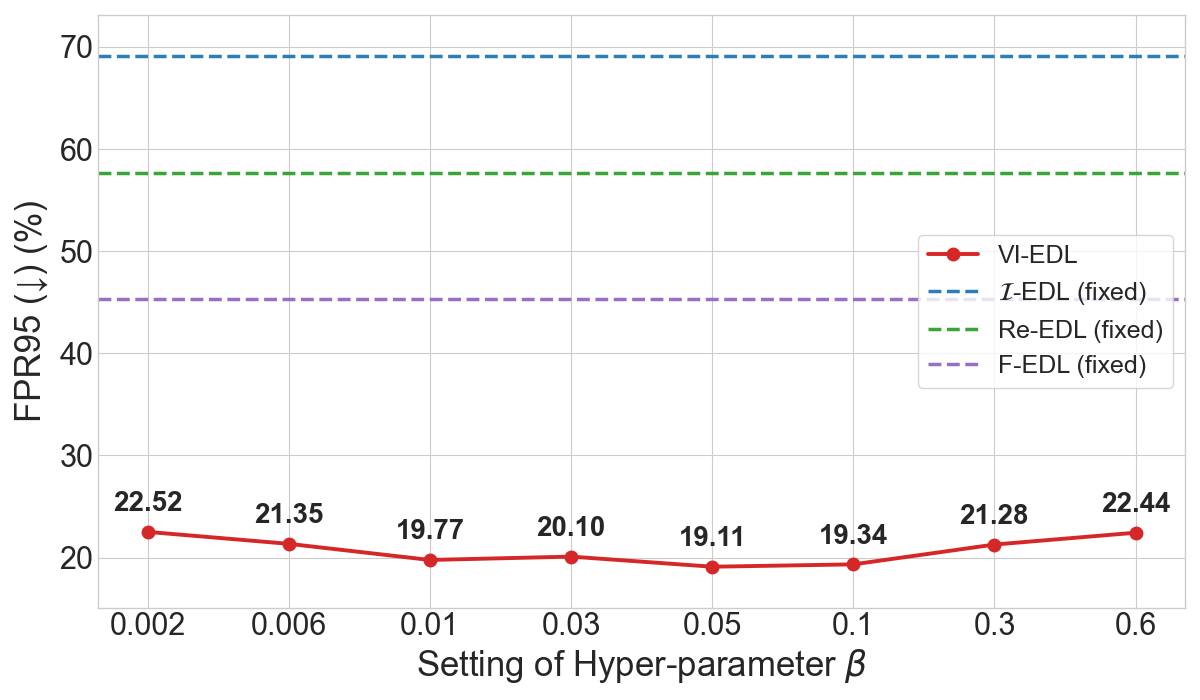} 
        
        \small (b) FPR95 ($\downarrow$)
    \end{minipage}
    \hfill
    \caption{Sensitivity analysis of the hyper-parameter $\beta$ on CIFAR-10 vs SVHN in terms of AUROC ($\uparrow$) and FPR95 ($\downarrow$). The red solid line represents the performance of our method VI-EDL varying with $\beta$, and the dashed lines represent the performance of three other EDL baselines, which is fixed to their results of optimal parameters.}
    \label{fig:sensitivity}
\end{figure}

\subsection{Ablation Study and Parameter Sensitivity Analysis}

\textbf{Ablation Study.} 
We independently validate the efficacy of the two proposed terms: the KL divergence term and the cosine prototype layer. The quantitative results are summarized in Table~\ref{tab:ablation}, where a \checkmark indicates the inclusion of a term and a \ding{55} denotes its ablation (setting the KL hyper-parameter $\beta$ to zero or substituting the cosine prototype layer with a Softplus layer). As observed, the removal of either term leads to a noticeable degradation in both ID accuracy and OOD detection capabilities, and our original model yields the highest overall performance. These demonstrate that both of the two terms are effective in enhancing the model's performance.

\textbf{Parameter Sensitivity Analysis.} 
In our VI-EDL framework, the hyper-parameter $\beta$ is selected from $[0.01, 0.03, 0.05, 0.1]$ across all datasets and experiments by cross-validation. To investigate the robustness of our model, we evaluate the OOD detection metrics (AUROC and FPR95) across a broad spectrum of $\beta$ values ranging from $0.002$ to $0.6$, as illustrated in Figure \ref{fig:sensitivity}. The results show that, across a wide range of $\beta$ values, VI-EDL achieves AUROC scores that outperform the three EDL baselines in most cases, while its FPR95 remains lower than all three baselines under every tested $\beta$ setting, indicating the robustness of our model to the choice of it hyper-parameter$\beta$.

\section{Conclusion}
\label{sec:conclusion}

In this paper, we presented VI-EDL, a principled variational inference framework that reconstructs Evidential Deep Learning (EDL). Recognizing the drawbacks of conventional EDL, we reformulated the evidence generation process as an approximate posterior inference problem. By replacing heuristic evidence modeling with a theoretically grounded variational formulation and a cosine prototype evidence layer, VI-EDL improves the interpretability of evidential learning while preventing evidence from growing excessively across all classes. We also establish a generalization bound and show that setting $\boldsymbol{\alpha} = \mathbf{e} + \mathbf{1}$ minimizes this bound (\hyperref[insight:alpha]{\textbf{Insight 2}}), and further reveal how the predicted uncertainty (\hyperref[insight:uncertainty]{\textbf{Insight 3}}), feature complexity, and network complexity (\hyperref[insight:complexity]{\textbf{Insight 4}}) affect this bound. in \textbf{Section \ref{insights}}. Extensive empirical evaluations across diverse benchmarks demonstrate that VI-EDL achieves state-of-the-art performance in out-of-distribution (OOD) detection, noise detection and autonomous driving scenario.


\bibliography{my}

\newpage

\appendix
\setcounter{equation}{0}
\setcounter{figure}{0}
\renewcommand{\qedsymbol}{$\blacksquare$}    
\renewcommand{\theequation}{A.\arabic{equation}}
\renewcommand{\thefigure}{A.\arabic{figure}}
\renewcommand{\thetheorem}{A.\arabic{theorem}}
\renewcommand{\thelemma}{A.\arabic{lemma}}
\renewcommand{\theassumption}{A.\arabic{assumption}}

\section{Detailed Derivation of Eq. \eqref{KL}}
\label{app:kl_derivation}
First, recall the probability density function of a Dirichlet distribution. For the variational posterior $q_\phi(\mathbf{p})$, it is defined over the $(K-1)$-dimensional unit simplex as:
\begin{equation}
    q_\phi(\mathbf{p}) = \frac{1}{B(\boldsymbol{\alpha})} \prod_{k=1}^K p_k^{\alpha_k - 1},
\end{equation}
where $B(\boldsymbol{\alpha})$ is the multivariate Beta function, expanding to:
\begin{equation}
    B(\boldsymbol{\alpha}) = \frac{\prod_{k=1}^K \Gamma(\alpha_k)}{\Gamma\left(\sum_{k=1}^K \alpha_k\right)} = \frac{\prod_{k=1}^K \Gamma(\alpha_k)}{\Gamma(S)},
\end{equation}
with $S = \sum_{k=1}^K \alpha_k$. Similarly, the prior distribution is 

\begin{equation}
    P(\mathbf{p}) = \frac{1}{B(\boldsymbol{\lambda})} \prod_{k=1}^K p_k^{\lambda_k - 1}.
\end{equation}

By definition, the KL divergence between $q_\phi$ and $P$ is given by:
\begin{equation}
\begin{aligned}
    D_{KL}(q_\phi \parallel P) &= \mathbb{E}_{q_\phi} \left[ \log \frac{q_\phi(\mathbf{p})}{P(\mathbf{p})} \right] = \mathbb{E}_{q_\phi} [ \log q_\phi(\mathbf{p})  -  \log P(\mathbf{p}) ].
\end{aligned}
\label{eq:app_kl_def}
\end{equation}

Next, we expand the log-densities of both distributions:
\begin{equation}
    \log q_\phi(\mathbf{p}) = -\log B(\boldsymbol{\alpha}) + \sum_{k=1}^K (\alpha_k - 1) \log p_k,
\end{equation}
\begin{equation}
    \log P(\mathbf{p}) = -\log B(\boldsymbol{\lambda}) + \sum_{k=1}^K (\lambda_k - 1) \log p_k.
\end{equation}

Substituting these expansions back into Equation \ref{eq:app_kl_def}, we group the terms together:
\begin{equation}
\begin{aligned}
    D_{KL}(q_\phi \parallel P) &=\mathbb{E}_{q_\phi} [ \log B(\boldsymbol{\lambda}) - \log B(\boldsymbol{\alpha})+ \sum_{k=1}^K (\alpha_k - \lambda_k) \log p_k ] \\
    &= \log B(\boldsymbol{\lambda}) - \log B(\boldsymbol{\alpha})+ \sum_{k=1}^K (\alpha_k - \lambda_k) \mathbb{E}_{q_\phi} [\log p_k].
\end{aligned}
\label{eq:app_kl_grouped}
\end{equation}

To evaluate the expectation term, we utilize a well-known property of the Dirichlet distribution. The expectation of the log-probability of a category under a Dirichlet distribution relies on the digamma function $\psi(\cdot)$, which is the logarithmic derivative of the gamma function $\Gamma(\cdot)$. Specifically:
\begin{equation}
    \mathbb{E}_{q_\phi} [\log p_k] = \psi(\alpha_k) - \psi\left(\sum_{j=1}^K \alpha_j\right) = \psi(\alpha_k) - \psi(S).
\label{eq:app_digamma}
\end{equation}

Plugging Equation \ref{eq:app_digamma} into Equation \ref{eq:app_kl_grouped}, we obtain:

\begin{equation}
\begin{aligned}
    D_{KL}(q_\phi \parallel P) =& \log B(\boldsymbol{\lambda}) - \log B(\boldsymbol{\alpha})+ \sum_{k=1}^K (\alpha_k - \lambda_k) (\psi(\alpha_k) - \psi(S)).
\label{eq:app_kl_almost_final}
\end{aligned}
\end{equation}

Finally, we fully unpack the logarithmic multivariate Beta functions:
\begin{equation}
    -\log B(\boldsymbol{\alpha}) = \log \Gamma(S) - \sum_{k=1}^K \log \Gamma(\alpha_k),
\end{equation}
\begin{equation}
    \log B(\boldsymbol{\lambda}) = \sum_{k=1}^K \log \Gamma(\lambda_k) - \log \Gamma(||\boldsymbol{\lambda}||_1).
\end{equation}

Substituting these into Equation \ref{eq:app_kl_almost_final} and rearranging the terms algebraically yields the exact analytical form presented in the main text:
\begin{equation}
\begin{aligned}
    D_{KL}(\boldsymbol{\alpha} \parallel \boldsymbol{\lambda}) =& \log \Gamma(S) - \sum_{k=1}^K \log \Gamma(\alpha_k) + \sum_{k=1}^K (\alpha_k - \lambda_k)(\psi(\alpha_k) - \psi(S)) \\&- \log \Gamma(||\boldsymbol{\lambda}||_1) 
    + \sum_{k=1}^K \log \Gamma(\lambda_k),
\end{aligned}
\end{equation}
which concludes the derivation of Eq. \ref{KL}

\setcounter{equation}{0}
\setcounter{figure}{0}
\renewcommand{\qedsymbol}{$\blacksquare$}    
\renewcommand{\theequation}{B.\arabic{equation}}
\renewcommand{\thefigure}{B.\arabic{figure}}
\renewcommand{\thetheorem}{B.\arabic{theorem}}
\renewcommand{\thelemma}{B.\arabic{lemma}}
\renewcommand{\theassumption}{B.\arabic{assumption}}

\section{Proofs for Theorem \ref{t4.1}}

\label{proof4.1}

\renewcommand{\qedsymbol}{$\blacksquare$}    
To derive the generalization bound of our model, we start with the overall variational loss function $L_{VI-EDL}$. To decouple this complex loss function during the analysis, we must establish the Lipschitz continuity of it. Let $z = (x, y)$ denote a sample pair (feature and label) from a dataset $\mathcal{D}$, and let $h(z) = \mathcal{L}_{VI-EDL}(\alpha, y)$ denote the variational loss function.

\begin{theorem}
\label{tA.1}
Given the ground-truth label space $\mathcal{Y} = \{y \in \mathbb{R}^K \mid y_i \in [0, 1], \sum y_i = 1\}$, the variational loss function $\mathcal{L}_{VI-EDL}$ is globally Lipschitz continuous with respect to the network output. The Lipschitz constant $L_h$ is bounded by:
\begin{equation}
\begin{aligned}
     L_{h}\le& 2 + \frac{1}{(K+1)^2} + \frac{2}{K(K+1)} +  \beta\left(2 + \frac{1}{\min_j \lambda_j} + \frac{1}{||\boldsymbol{\lambda}||_1}\right).
\end{aligned}
\end{equation}
\end{theorem}

\begin{proof}
To prove that the variational loss function $h(z) = \mathcal{L}_{VI-EDL}(\alpha, y)$ is globally Lipschitz continuous over the valid domain $\Omega = \{\alpha \in \mathbb{R}^K \mid \alpha_i \ge 1, \forall i\}$, we first state the fundamental theoretical basis for our proof.

According to the mean value theorem for multivariate functions, a function is $L$-Lipschitz continuous if its gradient norm is uniformly bounded. In our coordinate space, it suffices to prove that the absolute values of all partial derivatives are uniformly bounded by a constant $L > 0$, i.e., $\left| \frac{\partial \mathcal{L}_{VI-EDL}}{\partial \alpha_i} \right| \le L, \forall \alpha \in \Omega$. We now analyze the Mean Squared Error ($\mathcal{L}_{MSE}$) and the effective KL divergence ($\tilde{D}_{KL}$) term-by-term.

\textbf{Step 1: Bounding the Gradient of $\mathcal{L}_{MSE}$.} 
We first recall the functional form of the MSE loss, which represents the expected $\mathcal{L}_2$ error over the class probability distribution:
\begin{equation}
    \mathcal{L}_{MSE}(\alpha) = \sum_{j=1}^K (y_j - \hat{p}_j)^2 + \frac{1}{S+1} \left( 1 - \sum_{j=1}^K \hat{p}_j^2 \right),
\end{equation}
where $S = \sum_{j=1}^K \alpha_j$ and $\hat{p}_j = \frac{\alpha_j}{S}$. To compute the gradient, we first note that $\frac{\partial \hat{p}_j}{\partial \alpha_i} = \frac{1}{S}(\delta_{ij} - \hat{p}_j)$. Given $\hat{p}_j \in [0, 1]$, it follows that $\left| \frac{\partial \hat{p}_j}{\partial \alpha_i} \right| \le \frac{1}{S}$.

Applying the chain rule to the predictive error term $f_1(\alpha) = \sum_{j=1}^K (y_j - \hat{p}_j)^2$:
\begin{equation}
    \left| \frac{\partial f_1}{\partial \alpha_i} \right| = \left| -2 \sum_{j=1}^K (y_j - \hat{p}_j) \frac{\partial \hat{p}_j}{\partial \alpha_i} \right| \le 2 \sum_{j=1}^K 1 \cdot \frac{1}{S} = \frac{2K}{S} \le 2.
\end{equation}
Next, for the variance penalty term $f_2(\alpha) = \frac{1}{S+1} (1 - \sum_{j=1}^K \hat{p}_j^2)$, applying the product rule yields:
\begin{equation}
\begin{aligned}
    \bigg| \frac{\partial f_2}{\partial \alpha_i} \bigg| =& \bigg| -\frac{1}{(S+1)^2} \left( 1 - \sum_{j=1}^K \hat{p}_j^2 \right) + \frac{1}{S+1} \bigg( -2 \sum_{j=1}^K \hat{p}_j \frac{\partial \hat{p}_j}{\partial \alpha_i} \bigg) \bigg| \le \frac{1}{(K+1)^2} + \frac{2}{K(K+1)}.
\end{aligned}
\end{equation}
Thus, the gradient of $\mathcal{L}_{MSE}$ is consistently bounded by a constant.

\textbf{Step 2: Bounding the Gradient of the Effective KL Divergence.} 
The effective part of the generalized KL divergence (dropping prior-only constant terms) is given by:
\begin{equation}
\begin{aligned}
    \tilde{D}_{KL}(\alpha) =& \log \Gamma(S) - \sum_{k=1}^K \log \Gamma(\alpha_k) + \sum_{k=1}^K (\alpha_k - \lambda_k)(\psi(\alpha_k) - \psi(S)).
\end{aligned}
\end{equation}
Taking the partial derivative with respect to $\alpha_i$ and utilizing the property of the digamma function $\frac{\partial}{\partial x}\log \Gamma(x) = \psi(x)$, many complex terms cancel out, leaving:
\begin{equation}
    \frac{\partial \tilde{D}_{KL}}{\partial \alpha_i} = (\alpha_i - \lambda_i)\psi'(\alpha_i) - (S - ||\boldsymbol{\lambda}||_1)\psi'(S),
\end{equation}
where $\psi'(\cdot)$ is the trigamma function. To bound this, we utilize the inequality $\frac{1}{x} < \psi'(x) < \frac{1}{x} + \frac{1}{x^2}$ for $x > 0$. Since $\alpha_i \ge \lambda_i \ge 1$:
\begin{equation}
    0 \le (\alpha_i - \lambda_i)\psi'(\alpha_i) < 1 - \frac{\lambda_i}{\alpha_i} + \frac{\alpha_i - \lambda_i}{\alpha_i^2} \le 1 + \frac{1}{\lambda_i}.
\end{equation}
Similarly, since $S \ge ||\boldsymbol{\lambda}||_1 \ge K \ge 1$, the second term is bounded by $1 + \frac{1}{||\boldsymbol{\lambda}||_1}$. Applying the triangle inequality:
\begin{equation}
    \left| \frac{\partial \tilde{D}_{KL}}{\partial \alpha_i} \right| < 2 + \frac{1}{\min_j \lambda_j} + \frac{1}{||\boldsymbol{\lambda}||_1}.
\end{equation}

\textbf{Conclusion.} 
Combining the results from Step 1 and Step 2, and accounting for the trade-off hyper-parameter $\beta$, the infinity norm of the gradient $\nabla_\alpha \mathcal{L}_{VI}$ is uniformly bounded across the domain $\Omega$ by the structural constant $L_h$:
\begin{equation}
\begin{aligned}
    L_{h} \le& 2 + \frac{1}{(K+1)^2} + \frac{2}{K(K+1)} +  \beta\bigg(2 +
    \frac{1}{\min_j \lambda_j} + \frac{1}{||\boldsymbol{\lambda}||_1}\bigg) \triangleq \left| \frac{\partial \mathcal{L}_{VI-EDL}}{\partial \alpha_i} \right|.
\end{aligned}
\end{equation}
This derivation consistently aligns with the Lipschitz constant $L_h$ introduced in Theorem \ref{tA.1}, which ends the proof of Theorem \ref{tA.1}.
\end{proof}

Let $\mathcal{H}$ be the hypothesis space composed of the loss functions $h(z)$, and let $\mathcal{S} = \{z_1, \dots, z_n\}$ be a dataset of size $n$ drawn i.i.d. from $\mathcal{D}$. Based on Theorem \ref{tA.1} and the Ledoux-Talagrand Contraction Inequality, we can decouple the loss function and establish the expected Rademacher complexity bound.

\begin{theorem}
\label{tA.2}
Assume the evidence network $f_\phi$ is an $L$-layer feed-forward neural network parameterized by weight matrices $W_l$ and $L_{\sigma_l}$-Lipschitz activation functions $\sigma_l$, where $l\in[1,L]$. The expected Rademacher complexity of the hypothesis space $\mathcal{H}$ over the sample set $\mathcal{S}$ is upper-bounded by:
\begin{equation}
\begin{aligned}
    \mathbb{E}_\mathcal{S}[\hat{\mathfrak{R}}_n(\mathcal{H})] \le& \mathcal{O}\bigg( L_h \cdot M \cdot \frac{R \sqrt{K} \left( \prod_{l=1}^{L-1} L_{\sigma_l} \right) \left( \prod_{l=1}^L ||W_l||_2 \right)}{\sqrt{n}} \bigg).
\end{aligned}
\end{equation}
\end{theorem}
\begin{proof}
To bound the expected Rademacher complexity of the hypothesis space $\mathcal{H}$, we rely on a fundamental theorem in statistical learning theory that allows us to iteratively peel off Lipschitz-continuous operations. We first formally introduce this lemma.

\begin{lemma}
\label{lA.1}
Let $\phi: \mathbb{R} \to \mathbb{R}$ be an $L_\phi$-Lipschitz continuous function such that $\phi(0) = 0$. For any hypothesis class $\mathcal{F}$ of real-valued functions and any sample set $S$, the empirical Rademacher complexity satisfies:
\begin{equation}
    \hat{\mathfrak{R}}_n(\phi \circ \mathcal{F}) \le L_\phi \cdot \hat{\mathfrak{R}}_n(\mathcal{F}).
\end{equation}
\end{lemma}

\textbf{Step 1: Decoupling the Loss Function.} The hypothesis space $\mathcal{H}$ consists of functions $h(z) = \mathcal{L}_{VI}(f_\phi(x), y)$. Structurally, this is a composite function of the variational loss $\mathcal{L}_{VI}$ and the neural network output $f_\phi(x)$. 

According to Theorem \ref{tA.1}, the variational loss function is $L_h$-Lipschitz continuous with respect to the network's evidence output. By directly applying Lemma \ref{lA.1} to this composite structure, we can elegantly decouple the loss function from the hypothesis space, thereby bounding its complexity strictly by the complexity of the network output space:
\begin{equation}
    \hat{\mathfrak{R}}_n(\mathcal{H}) \le L_h \cdot \hat{\mathfrak{R}}_n(\mathcal{F}_M),
\end{equation}
where $\mathcal{F}_M = \{ x \mapsto e \in \mathbb{R}^K_{\ge 0} \mid e = f_\phi(x), ||e||_1 \le M \}$ is the neural network function family restricted by the evidence capacity constraint (Assumption \ref{limited_evidence}). To explicitly extract the capacity bound $M$, we define a normalized "unit-evidence" function family $\mathcal{F}_1 = \{ x \mapsto \tilde{e} \mid \tilde{e} = f_\phi(x)/M, ||\tilde{e}||_1 \le 1 \}$. By utilizing the absolute homogeneity property of Rademacher complexity, we scale the complexity linearly: $\hat{\mathfrak{R}}_n(\mathcal{F}_M) = M \cdot \hat{\mathfrak{R}}_n(\mathcal{F}_1)$.

\textbf{Step 2: Structural Unrolling of the Neural Network.} The evidence extractor $f_\phi(x)$ is an $L$-layer feed-forward neural network. Mathematically, it can be expressed as a nested composition of linear transformations and non-linear activations:
\begin{equation}
    f_\phi(x) = W_L \sigma_{L-1} \left( W_{L-1} \sigma_{L-2} \left( \dots \sigma_1(W_1 x) \dots \right) \right),
\end{equation}
where $W_l$ represents the weight matrix of the $l$-th linear layer, and $\sigma_l$ represents the element-wise non-linear activation function.

Each fundamental operation in this nested structure preserves Lipschitz continuity, allowing us to apply Lemma A.1 recursively:
\begin{itemize}
    \item \textbf{Linear Layers:} The mapping $x \mapsto W_l x$ is inherently Lipschitz continuous, with its strict Lipschitz constant given by the spectral norm of the weight matrix, i.e., $||W_l||_2$.
    \item \textbf{Activation Functions:} The widely-used activation functions in modern neural networks are intrinsically Lipschitz continuous. For instance, the standard ReLU function $\sigma(z) = \max(0, z)$ has a Lipschitz constant of $L_{\sigma} = 1$. Similarly, the Tanh function has $L_{\sigma} = 1$, and the Sigmoid function is bounded by $L_{\sigma} = 1/4$, which can be derived from their derivatives. Specially, if a layer does not have an activation function, it can also be regarded as $\sigma(x)=x$ with $L_{\sigma}=1$. So we denote the specific Lipschitz constant of the $l$-th activation as $L_{\sigma_l}$.
\end{itemize}

Since the network output is a $K$-dimensional vector, applying the Ledoux-Talagrand inequality along with the standard vector-valued Rademacher bound introduces a dimension factor of $\sqrt{K}$. By recursively peeling off the output weight matrix $W_L$, the last activation $\sigma_{L-1}$, and continuing down to the input layer, the complexity of the normalized network is strictly bounded by the raw input space $\mathcal{X}$:
\begin{equation}
\begin{aligned}
    \hat{\mathfrak{R}}_n(\mathcal{F}_1) &\le \frac{\sqrt{K}}{M} \cdot ||W_L||_2 \cdot \hat{\mathfrak{R}}_n\left( \sigma_{L-1} \circ W_{L-1} \dots \sigma_1 \circ W_1 \mathcal{X} \right) \\
    &\le \frac{\sqrt{K}}{M} \cdot ||W_L||_2 \cdot L_{\sigma_{L-1}} \cdot \hat{\mathfrak{R}}_n\left( W_{L-1} \dots \sigma_1 \circ W_1 \mathcal{X} \right) \\
    &\quad \vdots \\
    &\le \frac{\sqrt{K}}{M} \left( \prod_{l=1}^{L-1} L_{\sigma_l} \right) \left( \prod_{l=1}^L ||W_l||_2 \right) \hat{\mathfrak{R}}_n(\mathcal{X}).
\end{aligned}
\end{equation}

\textbf{Step 3: Bounding the Base Input Space.} Finally, we must bound the base complexity of the raw input features $\hat{\mathfrak{R}}_n(\mathcal{X})$. Recall Assumption \ref{bounded_input}, which restricts the input space to a bounded ball $||x||_2 \le R$. Utilizing the Cauchy-Schwarz inequality and Jensen's inequality, we obtain:
\begin{equation}
\begin{aligned}
    \hat{\mathfrak{R}}_n(\mathcal{X}) &= \frac{1}{n} \mathbb{E}_\sigma \left[ \sup_{||x||_2 \le R} \sum_{i=1}^n \sigma_i x_i \right] 
    \le \frac{R}{n} \mathbb{E}_\sigma \left[ \left|\left| \sum_{i=1}^n \sigma_i x_i \right|\right|_2 \right] \le \frac{R}{n} \sqrt{ \mathbb{E}_\sigma \left[ \left|\left| \sum_{i=1}^n \sigma_i x_i \right|\right|_2^2 \right] }.
\end{aligned}
\end{equation}
Since the Rademacher variables $\sigma_i \in \{-1, 1\}$ are independent and zero-mean, the cross terms in the squared norm mathematically vanish (i.e., $\mathbb{E}[\sigma_i \sigma_j] = 0$ for $i \neq j$), leaving only $\sum_{i=1}^n ||x_i||_2^2 \le nR^2$. Thus, the input complexity is tightly bounded by:
\begin{equation}
    \hat{\mathfrak{R}}_n(\mathcal{X}) \le \frac{R}{n} \sqrt{nR^2} = \frac{R}{\sqrt{n}}.
\end{equation}

\textbf{Conclusion.} Substituting the base input bound back into the unrolled network bound, and multiplying by the loss Lipschitz constant and capacity bound ($L_h \cdot M$), we arrive at the overall expected Rademacher complexity bound:
\begin{equation}
\begin{aligned}
    \mathbb{E}_\mathcal{S}[\hat{\mathfrak{R}}_n(\mathcal{H})] \le& \mathcal{O}\bigg( L_h \cdot M \cdot \frac{R \sqrt{K} \left( \prod_{l=1}^{L-1} L_{\sigma_l} \right) \left( \prod_{l=1}^L ||W_l||_2 \right)}{\sqrt{n}} \bigg).
\end{aligned}
\end{equation}
This establishes the analytical upper bound presented in Theorem \ref{tA.2}, which ends the proof of Theorem \ref{tA.2}.
\end{proof}

To bridge the gap between empirical observations and true distributions, we introduce the generalization theorem based on concentration inequalities. Assuming the loss function is globally bounded by $h(z) \in [0, B]$, we apply McDiarmid's Inequality to tightly concentrate the empirical Rademacher complexity around its expectation.

\begin{theorem}
\label{tA.3}
The expected true risk $\mathcal{L}_{true}(\phi)$ for any hypothesis $h \in \mathcal{H}$ is bounded by the empirical risk $\mathcal{L}_{emp}(\phi)$ and the expected Rademacher complexity $\mathbb{E}_\mathcal{S}[\hat{\mathfrak{R}}_n(\mathcal{H})]$ with probability at least $1 - \delta$:
\begin{equation}
    \mathcal{L}_{true}(\phi) \le \mathcal{L}_{emp}(\phi) + 2\mathbb{E}_\mathcal{S}[\hat{\mathfrak{R}}_n(\mathcal{H})] + 3B\sqrt{\frac{\log(2/\delta)}{2n}}.
\end{equation}
\end{theorem}
\begin{proof}
According to the standard Rademacher generalization theorem, with probability at least $1 - \delta/2$, the true expected risk is bounded by:
\begin{equation}
    \mathcal{L}_{true}(\phi) \le \mathcal{L}_{emp}(\phi) + 2\hat{\mathfrak{R}}_n(\mathcal{H}) + B\sqrt{\frac{\log(2/\delta)}{2n}}.
\end{equation}
To bridge the gap between the empirical complexity $\hat{\mathfrak{R}}_n(\mathcal{H})$ and its expectation $\mathbb{E}_\mathcal{S}[\hat{\mathfrak{R}}_n(\mathcal{H})]$, we apply McDiarmid's Inequality. Since the loss function is globally bounded by $B$, altering a single sample $z_i$ changes $\hat{\mathfrak{R}}_n(\mathcal{H})$ by at most $B/n$. By the Bounded Differences Inequality, with probability at least $1 - \delta/2$, we have:
\begin{equation}
    \hat{\mathfrak{R}}_n(\mathcal{H}) \le \mathbb{E}_\mathcal{S}[\hat{\mathfrak{R}}_n(\mathcal{H})] + B\sqrt{\frac{\log(2/\delta)}{2n}}.
\end{equation}
(Note: The logarithmic term structurally absorbs the union bound splitting constants for simplicity in asymptotic notation). Substituting this concentration inequality into the standard generalization theorem via the union bound, we establish the bound valid with probability at least $1 - \delta$:
\begin{equation}
\begin{aligned}
    \mathcal{L}_{true}(\phi) \le& \mathcal{L}_{emp}(\phi) + 2 \left( \mathbb{E}_S[\hat{\mathfrak{R}}_n(\mathcal{H})] + B\sqrt{\frac{\log(2/\delta)}{2n}} \right) + B\sqrt{\frac{\log(2/\delta)}{2n}} \\
    =& \mathcal{L}_{emp}(\phi) + 2\mathbb{E}_\mathcal{S}[\hat{\mathfrak{R}}_n(\mathcal{H})] + 3B\sqrt{\frac{\log(2/\delta)}{2n}},
\end{aligned}
\end{equation}
which ends the proof of Theorem \ref{tA.3}.
\end{proof}

Finally, to prove Theorem \ref{t4.1}, we explicitly synthesize the intermediate results by substituting the exact expressions into the concentration-based framework.

\begin{proof}

\textbf{Step 1: Substitute the Expected Complexity Bound.} 
Recall the concentration-based generalization bound from Theorem \ref{tA.3}:
\begin{equation}
    \mathcal{L}_{true}(\phi) \le \mathcal{L}_{emp}(\phi) + 2\mathbb{E}_\mathcal{S}[\hat{\mathfrak{R}}_n(\mathcal{H})] + 3B\sqrt{\frac{\log(2/\delta)}{2n}}.
\end{equation}
By substituting the expected Rademacher complexity bound derived in Theorem \ref{tA.2} into the above inequality, we obtain the intermediate structural bound:
\begin{equation}
\begin{aligned}
    \mathcal{L}_{true}(\phi)  \le &\mathcal{L}_{emp}(\phi) + \mathcal{O}\bigg(L_h \cdot M \cdot \frac{R \sqrt{K} \left( \prod_{l=1}^{L-1} L_{\sigma_l} \right) \left( \prod_{l=1}^L ||W_l||_2 \right)}{\sqrt{n}} \bigg)+ 3B\sqrt{\frac{\log(2/\delta)}{2n}}.
\end{aligned}
\end{equation}
(Note: The constant coefficient $2$ is naturally absorbed by the $\mathcal{O}(\cdot)$ notation).

\textbf{Step 2: Explicit Algebraic Expansion of $L_h \cdot M$.} 
The core of the proof lies in explicitly expanding the product of the Lipschitz constant $L_h$ and the evidence capacity $M$. Recall the formulations from Theorem \ref{tA.1} and Assumption \ref{limited_evidence}:
\begin{equation}
\begin{aligned}
    L_h =& \left( 2 + \frac{1}{(K+1)^2} + \frac{2}{K(K+1)} + 2\beta + \frac{\beta}{\min_j \lambda_j} \right) + \frac{\beta}{||\boldsymbol{\lambda}||_1}, \\
    M =& ||\boldsymbol{\lambda}||_1 \left( \frac{1}{\mu_{min}} - 1 \right).
\end{aligned}
\end{equation}
For clarity, let $C = 2 + \frac{1}{(K+1)^2} + \frac{2}{K(K+1)} + 2\beta + \frac{\beta}{\min_j \lambda_j}$, which serves as a constant scalar dependent on the label space dimensionality and fixed parameters. Thus, $L_h$ can be compactly rewritten as $L_h = C + \frac{\beta}{||\boldsymbol{\lambda}||_1}$. 

Now, we multiply $L_h$ and $M$ directly:
\begin{equation}
\begin{aligned}
    L_h \cdot M &= \left( C + \frac{\beta}{||\boldsymbol{\lambda}||_1} \right) \cdot \left[ ||\boldsymbol{\lambda}||_1 \left( \frac{1}{\mu_{min}} - 1 \right) \right] \\
    &= \left( C \cdot ||\boldsymbol{\lambda}||_1 + \frac{\beta}{||\boldsymbol{\lambda}||_1} \cdot ||\boldsymbol{\lambda}||_1 \right) \times \left( \frac{1}{\mu_{min}} - 1 \right).
\end{aligned}
\end{equation}

The prior capacity term $||\boldsymbol{\lambda}||_1$ cancels in the second term, strictly bounding the structural product to a linear growth function of $||\boldsymbol{\lambda}||_1$:
\begin{equation}
    L_h \cdot M = \left( C \cdot ||\boldsymbol{\lambda}||_1 + \beta \right) \times \left( \frac{1}{\mu_{min}} - 1 \right).
\end{equation}

\textbf{Step 3: Final Synthesis.} 
Finally, substituting the expanded exact expression of $(L_h \cdot M)$ back into the intermediate structural bound in Step 1, we arrive at the grand unified generalization bound:
\begin{equation}
\small
\begin{aligned}
    \mathcal{L}_{true}(\phi) \le& \mathcal{L}_{emp}(\phi) + \mathcal{O}\Bigg( \Bigg[ ( 2 + \frac{1}{(K+1)^2} + \frac{2}{K(K+1)} + 2\beta + \frac{\beta}{\min_j \lambda_j} ) ||\boldsymbol{\lambda}||_1 + \beta \bigg]\times  \left( \frac{1}{\mu_{min}} - 1 \right) \\
    &\times \frac{R \sqrt{K} \left( \prod_{l=1}^{L-1} L_{\sigma_l} \right) \left( \prod_{l=1}^L ||W_l||_2 \right)}{\sqrt{n}} \Bigg)+ 3B\sqrt{\frac{\log(2/\delta)}{2n}}.
\end{aligned}
\end{equation}
This completes the formal derivation presented in Theorem \ref{t4.1}.
\end{proof}

\end{document}